\documentclass[11pt,a4paper]{article}
\usepackage[left = 2cm, right = 2cm, top = 2cm, bottom  =2cm]{geometry}
\usepackage[english]{babel}
\usepackage[utf8]{inputenc}
\usepackage{amsmath, amssymb}
\usepackage{amsthm}
\usepackage{graphicx}
\graphicspath{{images/}}
\usepackage[bf,small]{caption}
\usepackage[numbers,sort&compress]{natbib}
\usepackage[ruled]{algorithm2e}
\usepackage[shortlabels]{enumitem}
\usepackage{booktabs}
\usepackage{longtable}
\usepackage{dsfont}
\usepackage{tikz}
\usetikzlibrary{shapes.geometric, arrows, positioning}
\usepackage{pdflscape}
\usepackage{placeins}
\usepackage{url}
\usepackage{color}

\newcommand{\bs}{\boldsymbol}
\newcommand{\E}{\mathds{E}}
\newcommand{\R}{\mathds{R}}

\newcommand{\best}[1]{\bfseries\boldmath #1}
\newcommand{\ic}[1]{#1}

\makeatletter
\newcommand{\Rg}{\mathcal{R}} % Main symbol: R (in calligraphic font)
\newcommand{\Rf}{\@ifnextchar*{\Rf@opt}{\Rf@noopt}}
\newcommand{\Rf@opt}[1]{\@ifnextchar\bgroup{\Rf@opt@child}{\Rg_{F}}}
\newcommand{\Rf@noopt}{\@ifnextchar\bgroup{\Rf@child}{\Rg_{F^m}}}
\newcommand{\Rf@opt@child}[1]{\@ifnextchar\bgroup{\Rf@opt@index{#1}}{\Rg_{F_{#1}}}}
\newcommand{\Rf@opt@index}[2]{\Rg_{{F_{#1}}#2}}
\newcommand{\Rf@child}[1]{\@ifnextchar\bgroup{\Rf@index{#1}}{\Rg_{F_{#1}^m}}}
\newcommand{\Rf@index}[2]{\Rg_{{F_{#1}}#2^m}}
\makeatother

\newcommand{\N}{\mathds{N}}

\DeclareMathOperator*{\argmin}{argmin}

\newtheorem{lemma}{Lemma}[section]

\newtheorem{defn}{Definition}[section]
\theoremstyle{remark}
\newtheorem{rmk}{Remark}[section]
\theoremstyle{plain}

\long\def\sfootnote[#1]#2{\begingroup%
\def\thefootnote{\fnsymbol{footnote}}\footnote[#1]{#2}\endgroup}
\def\bfootnote{\xdef\@thefnmark{}\@footnotetext}

\begin{document}

\thispagestyle{empty}
{\centering
\Large{\bf Handling Missing Data in Probabilistic Regression Trees} \vspace{.5cm}\\
\normalsize{ {\bf
Taiane Schaedler Prass${}^{\mathrm{a,}}$\sfootnote[1]{Corresponding author. This Version: \today},\let\thefootnote\relax\footnote{\hskip-.3cm$\phantom{s}^\mathrm{a}$ Graduate Program in Statistics - Universidade Federal do Rio Grande do Sul.}
Alisson Silva Neimaier${}^\mathrm{a}$, Guilherme Pumi${}^\mathrm{a}$
 \\
\let\thefootnote\relax\footnote{E-mails: taiane.prass@ufrgs.br (Prass); alissonneimaier@hotmail.com (Neimaier); guilherme.pumi@ufrgs.br (Pumi).}
\let\thefootnote\relax\footnote{ORCIDs: 0000-0003-3136-909X (Prass); 0000-0002-7524-0776 (Neimaier); 0000-0002-6256-3170 (Pumi).}\\
\vskip.3cm
}}
}

\begin{abstract}
Probabilistic regression trees are a smooth alternative to classical regression trees, consistent in the complete-data case, producing continuous predictions through probabilistic split assignments. This paper extends this framework by proposing three approaches to accommodate missing predictor values directly during tree construction, without a separate imputation step. These modifications are defined to preserve the fundamental probabilistic properties of the original methodology, including probability conservation and marginal compatibility, under arbitrary patterns of missing covariate values.  The proposed methods are evaluated on several real-world datasets exhibiting different levels of missingness and are compared with classical regression trees. In addition, a semi-synthetic study, in which missing values are injected completely at random, at random and not at random, compares the proposed strategies with impute-then-fit pipelines based on commonly applied imputation methods.\\[.2cm]
\end{abstract}
\section{Introduction}

Decision trees are a fundamental tool in statistical learning and machine learning, widely used for both classification and regression tasks \citep{blockeel2023,costa2023}. Their hierarchical structure partitions the predictor space through a sequence of recursive decisions, yielding models that are easy to interpret and implement, and tree-based models remain a central choice for heterogeneous tabular data, where gradient-boosted tree ensembles still mostly outperform deep learning models in supervised tasks \citep{borisov2024}. Among the various tree-based methods proposed in the literature, the Classification and Regression Trees (CART) algorithm of \citet{breiman84} remains one of the most widely used. CART recursively partitions the predictor space by selecting splits that locally optimize a prediction criterion. However, the resulting predictions are piecewise constant and may not adequately capture smooth relationships between predictors and the response variable \citep{irsoy2012,linero2018}.

Several authors have proposed probabilistic alternatives that replace deterministic splits by smooth transition mechanisms. Earlier works include the Smooth Transition Regression Tree (STR-Tree) model of \citet{medeiros2008} and the Soft Tree framework of \citet{irsoy2012}. In these approaches, observations are associated with tree nodes through probabilities rather than hard assignments, allowing predictions to vary continuously across the predictor space. Such probabilistic formulations reduce the discontinuities inherent to classical trees and provide greater flexibility for modeling complex regression functions. The same idea underlies fuzzy decision trees \citep{janikow1998}, in which observations are routed through soft memberships rather than hard assignments, and it has more recently been revisited in the so-called neural trees, which combine neural networks and decision trees \citep{li2022neural}; the review of \citet{blockeel2023} also discusses soft splits, in which instances are partially assigned to several branches. As noted by \citet{costa2023} in a recent survey, when predictions are obtained by averaging over the leaves, soft splits also restore a notion of locality that is absent in classical trees, since similar inputs lead to similar outputs. The Soft Tree framework \citep{irsoy2012} has also inspired recent work, such as the soft regression trees of \citet{consolo2026}, and similar soft splits are used in the randomized regression trees of \citet{blanquero2022}. However, despite their methodological differences, none of these models provides native mechanisms for handling missing covariates.

Within this literature, \citet{alkhoury2020} introduced Probabilistic Regression Trees (PRTrees), a framework that combines recursive partitioning with smooth probabilistic region assignments. In contrast to the other models mentioned above, PRTrees are accompanied by a consistency result, establishing convergence of the estimated regression function to the true conditional expectation under suitable conditions. This theoretical property, together with the ability to produce continuous predictions, makes PRTrees an attractive alternative to classical regression trees. However, as originally proposed, PRTrees do not provide a native mechanism for handling missing predictor values, limiting their applicability in settings where incomplete covariate information is present.

Despite these advantages, the original PRTree methodology is limited to fully observed predictor vectors. In its original formulation, no mechanism is provided for handling missing covariates during either model fitting or prediction. This limitation is particularly restrictive in practical applications, where incomplete predictor information is common. While missing values can always be addressed through complete-case analysis or external imputation procedures, such approaches require an additional preprocessing stage that is conceptually separate from the tree model itself.

Handling missing data remains one of the most challenging problems in applied statistics. Missing values may reduce effective sample size, increase uncertainty, and complicate the interpretation of predictive models \citep{missing_handbook20}. Recent theoretical work on supervised learning with missing values by \citet{LeMorvan2021} and \citet{josse2024} shows that impute-then-regress procedures can possess strong asymptotic optimality properties; for instance, imputing with a constant before learning is consistent when the missingness is non-informative \citep{josse2024}. However, it is also pointed out in \cite{LeMorvan2021} that conditional imputations may introduce discontinuities in the estimated regression function. For methods such as PRTrees, whose primary motivation is to produce smooth predictions through probabilistic assignments, such discontinuities may be undesirable.

For classical regression trees, several internal strategies have been proposed to mitigate the effects of missing predictors. In particular, the \texttt{rpart} implementation of CART \citep{rpart} employs surrogate splits \citep{breiman84}, which use alternative predictors when the primary splitting variable is unavailable. In \cite{josse2024}, several strategies that handle missing values directly in decision trees are compared and, after an empirical comparison, the ``missing incorporated in attribute'' strategy is recommended, since it accommodates both non-informative and informative missingness. However, such mechanisms are closely tied to the deterministic structure of CART and do not naturally extend to probabilistic tree models, where observations are associated with regions through smooth probability functions rather than hard binary decisions.

The absence of native missing-data mechanisms is therefore a practical obstacle to the broader use of probabilistic regression trees. The purpose of this paper is to address this limitation. Rather than proposing a new predictive paradigm or attempting to outperform existing missing-data methodologies, this paper develops an extension of the PRTree framework that can operate directly on incomplete predictor vectors. The proposed modifications are designed to preserve the probabilistic structure of PRTrees while eliminating the need for a separate preprocessing step devoted to handling missing covariates.

To this end, three alternative strategies are introduced. The first adopts a uniform-probability assignment whenever relevant predictor information is unavailable. The second relies on the observed coordinates only, producing deterministic compatibility assessments based on the available information. The third projects both observations and regions onto the observed dimensions and applies the original probabilistic smoothing mechanism in the resulting lower-dimensional space. Although conceptually different, all three approaches are constructed so as to preserve key probabilistic properties of the original framework, including probability conservation across splits and compatibility with marginal probability calculations.

A natural question is whether these extensions should be preferred over the simpler strategy of imputing missing values and subsequently fitting the original PRTree. The answer is far from straightforward. Perhaps the most delicate issue when employing imputation concerns the nature of the missing data mechanism --- whether missing completely at random (MCAR), missing at random (MAR), or missing not at random (MNAR). In practice, determining the appropriate mechanism is rarely a trivial matter, often being a point of discussion and a source of uncertainty. For MAR and MNAR settings, an additional layer of complexity emerges from the need to specify and model the missing data structure itself. Misspecification of either the missingness mechanism or its structural form can lead to biased or invalid imputations, which in turn may severely compromise the predictive performance of the downstream PRTree.

One of the most attractive features of the proposed framework is that its non-parametric nature bypasses imputation entirely, avoiding the pitfalls related to the specification of an imputation model (but see Remark~\ref{rmk:mechanism} regarding the missing-data mechanism). Despite this compelling motivation, the goal of this paper is not to establish superiority of the proposed framework over imputation-based alternatives, nor to provide an exhaustive comparative evaluation between the two paradigms. Nevertheless, to position the proposed strategies with respect to standard practice, Section~\ref{sec:sim} reports a focused semi-synthetic comparison with impute-then-fit pipelines under controlled MCAR, MAR and MNAR mechanisms. A comprehensive comparison between the two paradigms is left for future research. Rather, our aim is focused on providing a coherent extension of the PRTree framework that can operate directly on incomplete predictor vectors while preserving its probabilistic structure. From a practical point of view, the proposed methods offer an additional option for users who prefer to avoid the complexities and pitfalls of a separate missing-data treatment before fitting a probabilistic regression tree.

The original PRTree algorithm \citep{alkhoury2020} is available only in \texttt{Python}. The \texttt{PRTree} package \citep{PRTree_package}, now available on CRAN, makes the methodology available in the R ecosystem. The package implements both the original algorithm and the missing-data extensions proposed in this paper, with the computationally intensive components implemented in \texttt{FORTRAN} and \texttt{C}. A companion paper \citep{PRTree2026} provides implementation details and a complete description of the package. The main contributions of this paper are as follows.
\begin{itemize}
\item Extension of probabilistic regression trees to incomplete predictor vectors through three strategies (uniform-probability, partial-observation and dimension-reduced smoothing) that operate during both fitting and prediction, without a separate imputation step.

\item Verification that the three strategies preserve two fundamental properties of the original construction, probability conservation across splits and marginal compatibility, under arbitrary patterns of missing covariates (Section~\ref{sec:missing}).

\item Evaluation of the strategies on real datasets with naturally occurring missing values, which shows that, in the datasets with a large proportion of incomplete observations, the missing-data strategy is the most influential modelling choice, more so than the smoothing distribution or the proxy-selection criterion (Section~\ref{sec:aplic}).

\item Comparison of the strategies with impute-then-fit pipelines based on mean, $k$-nearest-neighbor, random-forest and multiple imputation in a semi-synthetic study under MCAR, MAR and MNAR mechanisms, identifying the settings in which each approach is preferable (Section~\ref{sec:sim}).
\end{itemize}

The remainder of this paper is organized as follows. Section \ref{sec:intro} reviews the original PRTree methodology. Section \ref{sec:missing} introduces the proposed missing-data extensions and establishes their probabilistic properties. Section \ref{sec:parameter} discusses parameter estimation. Section \ref{sec:aplic} presents the empirical evaluation on datasets with naturally occurring missing values, and Section~\ref{sec:sim} presents the semi-synthetic comparison with imputation-based pipelines under controlled missingness mechanisms. Finally, Section \ref{sec:conc} concludes with a discussion of the main findings and directions for future research.

\section{Probabilistic Regression Trees: The Base Model}\label{sec:intro}

Let $\bs X$ be a $p$-dimensional input random vector that lies almost surely in a compact subspace $\mathcal{X} \subset \R^p$, and let $Y$ be a response variable related to $\bs X$ through
\[
Y = f(\bs X) + \varepsilon,
\]
where the error term $\varepsilon$ satisfies $\E(\varepsilon) = 0$ and $\E(\varepsilon^2) < \infty$. These moment conditions ensure that the conditional expectation is well-defined and that the squared-error loss is meaningful. Although the special case $\varepsilon \sim \mathcal{N}(0, \tau^2)$ is often assumed for analytical convenience, our results do not require the normality assumption. In this framework, the function $f(\bs X) =\E(Y|\bs X)$ is referred to as the \emph{regression function} and represents the optimal predictor of $Y$ given $\bs X$ under the mean squared error criterion.

Tree-based methods provide a flexible approach to estimating $f$ without imposing strong parametric assumptions. The CART algorithm \citep{breiman84} builds regression trees via recursive binary partitioning of the feature space $\mathcal{X}$. Starting from the full dataset at the root node, the algorithm iteratively selects the optimal splitting variable $X_j$ and threshold $z$ that divide the data into two subsets according to the rule $X_j \leq z$ versus $X_j > z$. The partitioning continues until a stopping criterion is met --- such as a minimum node size, a threshold on error reduction, a maximum depth limit, or insufficient impurity decrease. The resulting tree partitions $\R^p$ into $M$ disjoint regions $\{\Rg_1, \cdots, \Rg_M\}$, with predictions given by the piecewise-constant function
\[
f_{\mathrm{CART}}(\bs X) = \sum_{m=1}^M c_m \, I(\bs X \in \Rg_m),
\]
where $I(\cdot)$ denotes the indicator function.

This greedy optimization procedure produces highly interpretable models that approximate $f(\bs X)$ through axis-aligned partitions of the input space. While computationally efficient, the resulting piecewise-constant predictions may struggle to capture smooth underlying relationships. PRTrees \citep{alkhoury2020} address this limitation by replacing the hard indicator functions in classical decision trees with smooth functions $\Psi$. The general prediction model takes the form
\begin{equation}\label{eq:PRTree}
f_{\mathrm{PR}}(\bs X; \Theta) = \sum_{m=1}^{M} \gamma_m \, \Psi(\bs X; \Rg_m, \bs\sigma),
\end{equation}
where $\Theta = (\{\Rg_m\}_{m=1}^{M}, \bs{\gamma}, \bs\sigma)$, comprises the partition regions $\{\Rg_m\}_{m=1}^{M}$, the region-specific weights $\bs{\gamma} \in \R^M$, and the noise vector $\bs\sigma \in \R_+^p$ controlling the smoothness of the soft assignments.

For any $\bs X \in \R^p$, $\Rg_m \subset \R^p$, and $\bs\sigma \in \R_+^p$, the association functions $\Psi$ are defined as
\begin{equation}\label{eq:Psi}
\Psi(\bs X; \Rg_m, \bs\sigma) = \biggl[\prod_{k=1}^p \sigma_k\biggr]^{-1} \int_{\Rg_m} \phi\biggl(\frac{v_1 - X_1}{\sigma_1}, \cdots, \frac{v_p - X_p}{\sigma_p} \biggr) \, d\bs{v},
\end{equation}
where $\phi$ is a probability density function belonging to $L^2$, continuously differentiable ($C^1$), and whose Fourier transform is nonzero on $\R^p$. The function $\Psi$ defines smooth, probabilistic memberships between points and regions, avoiding the abrupt boundaries of conventional trees. As a result, predictions vary continuously with $\bs X$ while retaining interpretability. The classical regression tree is recovered as the special case $\Psi(\bs X; \Rg_m, \bs\sigma) = I(\bs X \in \Rg_m)$.  In practice, selecting a suitable function $\phi$ is a nontrivial task. Prior knowledge of the noise distribution can help narrow down viable candidates for $\phi$ in \eqref{eq:Psi}, and the choice can be further refined using cross-validation.

The standard PRTree training algorithm --- which assumes fully observed feature vectors --- is described in detail in \cite{alkhoury2020}, where its key theoretical properties are also established, most notably consistency --- the guarantee that the estimated regression function converges to the true conditional expectation as sample size increases.
The consistency of PRTree is a remarkable result --- not known for earlier probabilistic tree methods such as STR-Trees or Soft Trees --- making PRTree particularly attractive for regression tasks. However, the original formulation provides no mechanism for handling missing covariates, restricting its use to complete datasets. In the next sections, we present an extension that preserves the probabilistic structure and interpretability of PRTrees while accommodating missing covariates directly.

\section{Handling Missing Values in PRTrees}\label{sec:missing}

In what follows, we present the proposed modifications to the PRTree algorithm for missing data handling. Each proposed modification is constructed in such a way as to guarantee the production of well-defined association values for any pattern of missing covariates. Well-definedness of the association values is an algebraic property of the construction and does not, by itself, imply any statistical guarantee under a given missingness mechanism (see Remarks~\ref{rmk:scope} and~\ref{rmk:mechanism}).

\subsection{Notation and Setup}

To fix notation, let $S \subseteq \{1,\cdots,p\}$ denote an arbitrary subset of indices. For any $\bs X \in \R^p$, denote by $\bs X_{|S}$ the subvector containing only the coordinates indexed by $S$. For any rectangular region $\Rg_m = \prod_{j = 1}^p \Rg_{mj} \subset \R^p$, denote by $\Rg_{m|S} = \prod_{j \in S} \Rg_{mj}$ its projection onto these coordinates, and for any vector $\bs \sigma \in \R_+^p$, let $\bs \sigma_{|S}$ be its restriction to the components in $S$.  A notational convention used throughout the paper requires explicit mention: $I(A)$ denotes the indicator function of event $A$ and, by convention, is always evaluated first in any expression where it appears. This convention ensures unambiguous interpretation of formulas involving products of indicators with other mathematical objects.

The key to our missing-data extension is that the PRTree smoothing function~\eqref{eq:Psi} naturally induces a probability measure over the set of regions $\{\Rg_m\}_{m=1}^M$, conditioned on the observed input $\bs X$:
\[
P(\Rg_m | \bs X) := \Psi(\bs X; \Rg_m, \bs\sigma), \quad 1 \leq m \leq M.
\]
Let $\Psi^\ast$ denote the adapted smoothing function in the presence of missing values. The corresponding probability measure is then $P^\ast(\Rg | \bs X) := \Psi^\ast(\bs X; \Rg, \bs\sigma)$. Any valid extension to handle missing values must preserve two fundamental properties of the original construction:
\begin{enumerate}[label = {\bf P\arabic*:}]
\item \textbf{Marginal probability compatibility for unbounded coordinates.}
If a region $\Rg_m$ does not impose any restriction on coordinate $k$, that is, $\Rg_{mk} = \R$ for some $1\le k\le p$, then
\[
P(\Rg_m | \bs X) = P(\Rg_{m|S} | \bs X_{|S}), \quad \text{with } S = \{1,\cdots,p\}\setminus\{k\}.
\]
\item \textbf{Probability conservation under splitting.}
If a parent region $\Rf*$ is partitioned into two child regions $\Rf*{L}$ and $\Rf*{R}$, then the law of total probability must hold:
\[
P(\Rf*{L} | \bs X) + P(\Rf*{R} | \bs X) = P(\Rf* | \bs X), \quad \text{for all } \bs X \in \mathcal{X}.
\]
\end{enumerate}

Property \textbf{P1} ensures that coordinates which are unrestricted (i.e., have not been bounded by previous splits) do not affect the probability computation. Consequently, if a coordinate is unbounded, missing values in that coordinate cannot influence the probability assigned to the region. Property \textbf{P2} guarantees that the total probability mass assigned to a region is exactly redistributed among its children after a split, maintaining coherence as the tree grows.

\begin{rmk}[Scope of P1 and P2]\label{rmk:scope}
Properties \textbf{P1} and \textbf{P2} are algebraic constraints on the adapted association function $\Psi^\ast$: they guarantee that, for every pattern of missing covariates, the association values are well defined, non-negative, and coherently redistributed along the tree. They do \emph{not} imply statistical consistency of the PRTree estimator when covariates are missing. The consistency result in \cite{alkhoury2020}, recalled at the end of Section~\ref{sec:parameter}, applies to the complete-data case, and its extension to the estimators proposed here is not proved in this paper. The theoretical obstacles to such an extension are discussed in Section~\ref{sec:conc}.
\end{rmk}

\subsection{Three Strategies for Handling Missing Values}

Our approach introduces three alternative strategies for handling missing values, controlled in the \texttt{PRTree} package by the parameter \texttt{fill\_type} (taking values 0, 1 or 2). Each strategy specifies how the association function $\Psi^\ast$ should be computed when some coordinates of $\bs X$ are missing:
\begin{enumerate}[label = \texttt{\arabic*:}, leftmargin = *]
\addtocounter{enumi}{-1}
\item \textbf{Uniform probability method.} When any value is missing, this strategy ignores all observed information and assigns equal weight to both child regions. It is the most conservative approach, requiring no assumption on the relation between observed and missing coordinates, at the cost of discarding potentially useful data.

\item \textbf{Partial observation approach.} This strategy performs a ``hard'' assignment based on the observed coordinates: it returns 1 if the observed values are compatible with the region and 0 otherwise. Smoothing via $\Psi^\ast$ is applied only to fully observed cases, offering a balance between robustness and informativeness.

\item \textbf{Smoothed projection technique.} Here, both the observation and the region are projected onto the dimensions corresponding to the observed coordinates, and the adapted smoothing function $\Psi^\ast$ is applied in this reduced space. This makes full use of the available information while preserving the smooth structure of PRTrees.
\end{enumerate}
When the coordinate used in a split is missing, the three strategies divide the association of the parent node equally between its two children (see Lemma~\ref{lem:product}). This rule is not new: it was proposed by \citet{quinlan1989} for classical decision trees and adopted by \citet{janikow1998} for fuzzy decision trees.

This flexible framework allows practitioners to choose an approach aligned with their application's priorities --- favoring simplicity, balanced trade-offs, or maximal informativeness --- without compromising the integrity of probability estimation. The adapted function $\Psi^\ast$ (and, consequently, the corresponding probability measure $P^\ast$) is defined recursively as follows.

\begin{defn}
Let $\Psi$ be defined by \eqref{eq:Psi}. Given $\bs X \in \R^p$ and any region $\Rg \subseteq \R^p$:

\noindent\textbf{Base case:} if $\Rg = \R^p$ (root node), then
\[
\Psi^\ast(\bs X; \Rg, \bs\sigma) = 1.
\]

\noindent \textbf{Recursive step:}  if $\Rg \in \{\Rf*{L}, \Rf*{R}\}$ is a child region resulting from the split of a parent region $\Rf*$, first define the set of indices $S(\bs X, \Rg)$ as follows:
\[
S(\bs X, \Rg) := \bigl\{ j \in \{1,\cdots,p\} : \Rf*{L}{j} \subsetneq \R \text{ and } X_j \text{ is non-missing}\bigr\},
\]
i.e., the indices $j$ such that the corresponding variable has been used in a previous split and $X_j$ is non-missing. Next, define a proxy function $H$ as:
\begin{equation}\label{eq:missing_handling}
H(\bs X; \Rg, \bs\sigma) =
\begin{cases}
\Psi(\bs X; \Rg, \bs\sigma), & \text{if } \bs X \text{ is fully observed}, \\[4pt]
1, & \text{if } \texttt{fill\_type} = 0 \text{ or } S(\bs X, \Rg)  = \emptyset, \\[4pt]
I(\bs X_{|S(\bs X, \Rg)} \in \Rg_{|S(\bs X, \Rg)}), & \text{if } \texttt{fill\_type} = 1 \text{ and } S(\bs X, \Rg) \neq \emptyset, \\[4pt]
\Psi(\bs X_{|S(\bs X, \Rg)}; \Rg_{|S(\bs X, \Rg)}, \bs\sigma_{|S(\bs X, \Rg)}), & \text{if } \texttt{fill\_type} = 2  \text{ and } S(\bs X, \Rg) \neq \emptyset.
\end{cases}
\end{equation}
Then $\Psi^\ast$ is given recursively by
\[
\Psi^\ast(\bs X; \Rg, \bs\sigma) = \frac{H(\bs X; \Rg, \bs\sigma)\Psi^\ast(\bs X; \Rf*, \bs\sigma)}{H(\bs X; \Rf*{L}, \bs\sigma) + H(\bs X; \Rf*{R}, \bs\sigma)}I\big(\Psi^\ast(\bs X; \Rf*, \bs\sigma)> 0\big).
\]
\end{defn}

In words, the definition reads as follows. At the root, every observation has association equal to $1$. Each time a node is split, the association of the parent node is divided between its two children in proportion to the proxy $H$, which measures how compatible the observation is with each child. If the observation is complete, $H$ is the usual association function $\Psi$, and the original PRTree is recovered. If it is incomplete, $H$ depends on \texttt{fill\_type}: it equals $1$ for both children (equal split) under \texttt{fill\_type} = 0; it equals $1$ or $0$ according to whether the \emph{observed} coordinates lie in the child region under \texttt{fill\_type} = 1; and it equals $\Psi$ computed only on the observed coordinates under \texttt{fill\_type} = 2. In all cases, only the coordinates that have already been used to split some ancestor of the node, and that are observed for the given observation, enter the computation (the set $S(\bs X,\Rg)$).

The denominator of the recursion is strictly positive whenever the indicator equals one, so $\Psi^\ast$ is well defined for every $\bs X$ and every \texttt{fill\_type}. Indeed, if $\Psi^\ast(\bs X;\Rf*,\bs\sigma)>0$, then $H(\bs X;\Rf*,\bs\sigma)>0$ (it equals $1$ at the root and, otherwise, multiplies the numerator). If $\bs X$ is incomplete and either \texttt{fill\_type}$=0$ or the splitting coordinate $j$ is missing, both children receive the same positive proxy. Otherwise, the children cover $\Rf*$ along the observed coordinate $j$: for complete $\bs X$ or \texttt{fill\_type}$=2$, additivity of $\Psi$ and \textbf{P1} give a denominator equal to $H(\bs X;\Rf*,\bs\sigma)$, and for \texttt{fill\_type}$=1$ the indicator of the parent equals $1$, so $X_j$ lies in at least one child interval and at least one proxy equals $1$. In particular, $\Psi^\ast(\bs X;\Rf*{L},\bs\sigma)+\Psi^\ast(\bs X;\Rf*{R},\bs\sigma)=\Psi^\ast(\bs X;\Rf*,\bs\sigma)$, that is, \textbf{P2} holds.

From the definition of $S$, note that $S(\bs X, \Rf*{L}) = S(\bs X, \Rf*{R})$, and $S(\bs X, \Rf*) \subseteq S(\bs X, \Rf*{L})$, with equality if, and only if, either the splitting coordinate $j$ was already in $S(\bs X,\Rf*)$ or $X_j$ is missing. This leads to a fundamental property shared by all strategies: if a partition $\Rf* = \Rf*{L} \cup \Rf*{R}$ is created by splitting on feature $j$, then $H(\bs X; \Rf*{L}, \bs\sigma) = H(\bs X; \Rf*{R}, \bs\sigma)$ whenever $X_j$ is missing. When missing values occur in other coordinates but $X_j$ is observed, each strategy behaves as previously described: \texttt{fill\_type = 0} ignores all observed information, \texttt{fill\_type = 1} performs a hard assignment based on observed coordinates, and \texttt{fill\_type = 2} projects onto the observed dimensions before applying $\Psi$. All three methods provide well-defined probability estimates for any missing-data pattern and reduce to the standard PRTree algorithm when no values are missing.   Lemma~\ref{lem:product} gives a characterization of $\Psi^\ast$, useful for practical implementation when there are missing values in the coordinates used to build the tree.

\begin{lemma}[\bf Product representation]\label{lem:product}
Fix a node $\Rg_m$ in the tree and $\bs X \in \R^p$. Denote by $\Rf$ the parent node of $\Rg_m$ and by $\Re_m$ the set of internal nodes on the path from the root (including) to $\Rg_m$. Let  $S := S(\bs X, \Rg_m)$ and $\Re_{m|S} \subseteq \Re_m$ be the subset of nodes that were split using a coordinate that belongs to $S$. If $\Psi^\ast(\bs X;\Rf,\bs\sigma) \neq 0$, then
\begin{equation} \label{eq:Pstar-product}
\Psi^\ast(\bs X;\Rg_{m},\bs\sigma) \; =  2^{-|\Re_m \setminus \Re_{m|S}|}\!\!\!\prod_{\Rf*\in\Re_{m|S}}\!\!\!\frac{H(\bs X;\Rg_{c(\Rf*)},\bs\sigma)}{H(\bs X;\Rf*{L},\bs\sigma) + H(\bs X;\Rf*{R},\bs\sigma)},
\end{equation}
where, for any $\Rf* \in\Re_m$ with children $\Rf*{L}$ and $\Rf*{R}$,  $c(\Rf*) \in\{F_L, F_R\}$ is the child index on the path to $\Rg_m$.
\end{lemma}
\begin{proof}
The result follows immediately by recursively applying the definition of $\Psi^\ast$.
\end{proof}

Lemma~\ref{lem:product} states that the association of an observation with a node is a product, along the path from the root to that node, of the proportions in which each ancestor divides its association between its children. Ancestors split on a coordinate that is missing for the observation contribute the factor $1/2$, while the remaining ancestors contribute the ratio of proxies. For example, let $p=2$, suppose the root is split on $X_1$ and its left child is split on $X_2$, and consider an observation with $X_1$ observed and $X_2$ missing. Under \texttt{fill\_type} = 2, the first split is computed exactly as in the complete case (by \textbf{P1}, since $X_2$ is not yet restricted); under \texttt{fill\_type} = 1, the whole association of the root is assigned to the child that contains $X_1$; and under \texttt{fill\_type} = 0, it is divided equally. For all three strategies, the second split divides the association of the left child equally between its two children, because the splitting coordinate $X_2$ is missing.

An immediate consequence of Lemma~\ref{lem:product} is that, for any strategy adopted (that is, any \texttt{fill\_type}), if $p = 1$ and $X$ is missing, $\Psi^\ast(X, \Rg_m, \sigma) = 2^{-d(m)}$, where $d(m)$ is the depth of $\Rg_m$.  In the particular case when $P(\Rf* | \bs X) = \Psi(\bs X; \Rf*, \bs\sigma)$ can be factorized as
\begin{equation*}
P(\Rf* | \bs X) = \Psi(\bs X; \Rf*, \bs\sigma) = \prod_{j=1}^p \Psi(X_j; \Rg_{Fj}, \sigma_j) = \prod_{j=1}^p P(\Rg_{Fj}| X_j),
\end{equation*}
then the same holds upon replacing $\Psi$ with $\Psi^\ast$ and the child probabilities satisfy
\[
P^\ast(\Rf*{L} | \bs X)  = \frac{P^\ast(\Rf*{L}{j} | X_j)P^\ast(\Rf* | \bs X)}{P^\ast(\Rf*{L}{j} | X_j) + P^\ast(\Rf*{R}{j} | X_j)} \quad  \text{and} \quad
    P^\ast(\Rf*{R} | \bs X) =  \frac{P^\ast(\Rf*{R}{j} | X_j)P^\ast(\Rf* | \bs X)}{P^\ast(\Rf*{L}{j} | X_j) + P^\ast(\Rf*{R}{j} | X_j)},
\]
whenever $P^\ast(\Rf* | \bs X) > 0$. This means that, at each step, we only need to compute the marginal probabilities for the splitting feature $j$ and then scale them by the probability of the parent node. This fact speeds up computations.

\begin{rmk}[Missingness mechanisms]\label{rmk:mechanism}
The three strategies do not require a separate imputation step, but this should not be confused with robustness to an arbitrary missingness mechanism \citep{rubin1976}. When the probability that a covariate is missing depends on its own unobserved value (MNAR), the missingness pattern carries information about the covariate and, possibly, about the response, and none of the three strategies models this dependence: \texttt{fill\_type} = 0 ignores the observed information altogether, while \texttt{fill\_type} = 1 and 2 use only the observed coordinates. Moreover, whenever a node is split on a coordinate that is missing, all strategies divide the association equally between the two children, irrespective of how the training data are distributed between them. Hence, as for imputation-based pipelines, the predictions may be biased under MAR and, especially, under MNAR mechanisms. The datasets of Section~\ref{sec:aplic} have naturally occurring missing values with unknown mechanism; the behavior of the three strategies under controlled MCAR, MAR and MNAR mechanisms is investigated empirically in Section~\ref{sec:sim}.
\end{rmk}

\section{Parameter estimation}\label{sec:parameter}

Given a training sample $\{(\bs X_i, Y_i)\}_{i=1}^n$, with $\bs X \in \R^p$, $Y \in \R$, and in accordance with the empirical risk minimization principle with a quadratic loss, the estimation procedure for probabilistic regression trees aims at finding the parameters $\Theta$ solutions of
\begin{equation}\label{eq:tree}
\argmin_{\Theta \in \Xi}\Biggl\{\sum_{i=1}^{n} \biggl( Y_{i} - \sum_{m=1}^{M} \gamma_m P_{im} \biggr)^2\Biggr\}, \quad \text{with } P_{im} := \Psi^\ast(\bs X_i; \Rg_m, \bs\sigma),
\end{equation}
where
\[
\Xi = \Bigl\{\bigl(\{\Rg_m\}_{m=1}^{M}, \bs{\gamma}, \bs\sigma\bigr): M \in \N\backslash\{0\}, \Rg_m \subseteq \R^p, \bs\gamma \in \R^M, \bs\sigma \in \R^p_{+}\Bigr\}.
\]
The $n \times M$ matrix $P$, with entries $P_{im}$, thus encodes the relations between each training example $\bs X_i$ and each region $\Rg_m$. It is such that $0 \leq P_{im} \leq 1$ and $\sum_{m=1}^{M} P_{im} = 1$, for all $1 \leq i \leq n$. The estimation of the different parameters in $\Theta$ alternates in between region and weight estimates, as in standard regression trees, until a stopping criterion is met. During this process, the number of regions is increased by one at each loop and the matrix $P$ and the weights $\bs \gamma$ are gradually updated. The vector $\bs\sigma$ can either be based on a priori knowledge or be estimated through a grid search on a validation set.

\paragraph{Estimating $\bs{\gamma}$.} Given the regions $\{\Rg_m\}_{m=1}^{M}$ and the vector $\bs\sigma$, minimizing \eqref{eq:tree} with respect to $\bs{\gamma}$ leads to the least square estimator
\begin{equation}\label{eq:minimize_PRTree}
\hat{\bs{\gamma}} = \argmin_{\bs{\gamma} \in \R^M} \Biggl\{ \sum_{i=1}^n \biggl( Y_i - \sum_{m=1}^M \gamma_m P_{im} \biggr)^2 \Biggr\} = \argmin_{\bs{\gamma} \in \R^M} \Bigl\{ || \bs Y - P \bs \gamma||^2 \Bigr\},
\end{equation}
where $\bs Y = (Y_1, \cdots, Y_n)'$. If $P'P$ is not singular, the solution is unique and it is given by
\[
\hat{\bs{\gamma}} = (P' P)^{-1} P' \bs{Y}.
\]

\paragraph{Estimating $\{\Rg_m\}_{m=1}^{M}$.} Assume that $M$ regions, referred to as current regions, have already been identified, meaning that the current tree has $M$ leaves. As in standard regression trees, each current region $\Rg_m$, $1 \leq m \leq M$, can be decomposed into two sub-regions $\Rf*{L}$ and $\Rf*{R}$ with respect to a coordinate $1 \leq j \leq p$ and a splitting point $t$ (threshold) that minimizes \eqref{eq:tree}. Each possible split leads to the update of $P$, that now belongs to $\R^{n\times (M+1)}$ (the space of $n$ by $M+1$ real matrices), and $\bs{\gamma}$, that now belongs to $\R^{M+1}$. The corresponding $\hat{\bs{\gamma}}$ is obtained through \eqref{eq:minimize_PRTree}. To make explicit the dependence of these quantities on the region, the splitting variable and the threshold, we shall use the notation $P^{(m)}(j,t)$ and $\hat{\bs\gamma}^{(m)}(j,t)$. The best split for the current region $\Rg_{m}$ solves
\[
\argmin_{ (j,t) \in \mathcal{J} \times \mathcal{T}^{(m)}_{j}}\Biggl\{ \sum_{i=1}^n \biggl(Y_i - \sum_{k=1}^{M+1} \hat{\gamma}_k^{(m)}(j,t) P_{ik}^{(m)}(j,t) \biggr)^2 \Biggr\},
\]
where $\mathcal{J} = \{j \in \N : 1 \leq j \leq p\}$ and $\mathcal{T}^{(m)}_{j}$ denotes the set of splitting points for region $\Rg_m$ and variable $j$ (more precisely, the set of middle points of the observations from $\Rg_m$ projected on the $j$th coordinate). At each step, the split that leads to the smallest mean square error is selected.

Note that, unlike in CART, a region $\Rg_m$ in a PRTree need not contain any observations. This is because $\Psi^\ast(\bs X; \Rg_m, \bs\sigma)$ defines a conditional probability (or association measure) that does not require $\bs X$ to actually fall within the region -- it only measures how compatible $\bs X$ is with that region. Nevertheless, the observed data still provide a natural set of candidate thresholds for splitting, and we adopt the standard midpoint rule for computational efficiency and consistency with tree-based methods.

\paragraph{Stopping.} The process terminates when the maximum number of regions ($M_{\max}$) is reached, the reduction in loss $\Delta L$ falls below a threshold $\epsilon$, or other stopping criteria (e.g., depth, minimum node size) are met.\par \vspace{1\baselineskip}

It is shown in \cite{alkhoury2020} that the PRTree learned from a training set of size $n$, with no missing data, denoted $\hat{f}_{\mathrm{PR}}^{(n)}$, is consistent in the sense that
\[
\lim _{n \to\infty} \E\left(\left|\hat{f}_{\mathrm{PR}}^{(n)}(\bs X)-\E(Y|\bs X)\right|^2\right)=0.
\]
The consistency is a desirable theoretical development, but the proof of such result is quite involved even in the standard framework without missing data. Extending the results to account for the missing data handling mechanism in \eqref{eq:missing_handling} is a non-trivial task and will be explored in future works. We stress that this consistency result concerns the complete-data case; whether it extends to the estimators proposed in Section~\ref{sec:missing} is not established here (see Remark~\ref{rmk:scope}), and the main obstacles for each strategy are discussed in Section~\ref{sec:conc}.

\section{Empirical Application}\label{sec:aplic}

The objective of this empirical study is to evaluate the practical behavior of the proposed missing-data extensions in real regression problems containing naturally occurring missing covariates. Rather than establishing superiority over alternative missing-data methodologies, the goal is to assess whether probabilistic regression trees remain a viable modeling option when incomplete predictor information is handled internally by the algorithm. We consider a collection of benchmark regression datasets exhibiting heterogeneous sample sizes, predictor dimensions, and missing-data patterns. The proposed extensions are evaluated through cross-validation and compared with classical CART regression trees, which provide a natural baseline for assessing the practical impact of probabilistic region assignments and internal missing-data handling. Comparisons with imputation-based pipelines, under controlled missingness mechanisms, are presented separately in Section~\ref{sec:sim}.

Particular attention is given to the influence of the missing-data strategy, the proxy selection criterion, and the smoothing distribution used by the probabilistic model. The empirical analysis focuses on predictive accuracy and robustness across datasets with naturally occurring missing values. The objective is not to identify a universally superior specification, but rather to investigate how the proposed extensions behave under different modeling choices and whether they provide a practical alternative for users who prefer not to perform a separate missing-data treatment before fitting a probabilistic regression tree.

\subsection{Main empirical question}
\label{subsec:main-empirical-question}

The empirical study addresses two questions. The first is whether PRTrees remain competitive with classical CART regression trees when applied to datasets containing naturally occurring missing covariate values. This comparison is based on predictive performance measured through RMSE and relative predictive error.

The second question concerns the effect of the main modeling choices available in the proposed methodology. In particular, the experiments investigate how predictive performance varies with the missing-data extension (\texttt{fill\_type}), the proxy-selection criterion (\texttt{proxy\_crit}), and the smoothing density specification $\phi$. The objective is to identify which combinations of these components are associated with the best predictive performance across the datasets considered in the study.

\subsection{Data Acquisition and Preparation}

The empirical study was conducted using real-world regression datasets obtained primarily from the UCI Machine Learning Repository. Candidate datasets were initially identified through the \texttt{ucimlrepo} R package \citep{ucimlrepo}, which provides programmatic access to dataset metadata and facilitates large-scale screening of regression problems. To ensure a diverse yet computationally manageable benchmark, datasets were filtered according to basic characteristics such as sample size, predictor dimension, and response type. An additional benchmark dataset, \textit{Ozone}, available in the \texttt{mlbench} package \citep{mlbench}, was also included in the analysis.

The final collection consists of five regression datasets: \textit{Communities and Crime} (\emph{Crime}), \textit{Productivity Prediction of Garment Employees} (\emph{Garment}), \textit{Auto MPG}, \textit{Automobile}, and \textit{Ozone}. These datasets exhibit substantial heterogeneity in both sample size and predictor dimension, while all containing naturally occurring missing values in the predictor variables. This diversity provides a realistic setting for evaluating probabilistic regression trees under incomplete covariate information.

Prior to model fitting, observations with missing response values were removed whenever necessary. Missing values in the predictor variables were preserved and handled directly by the proposed methods. Categorical predictors were excluded from the analysis so that all models were fitted using numerical covariates only. This preprocessing step avoids the introduction of additional variability associated with categorical encoding schemes and allows the empirical evaluation to focus exclusively on the effects of missing predictor information. Consequently, the empirical evaluation is restricted to numerical predictors. This restriction is inherited from the base model: neither the original PRTree \citep{alkhoury2020} nor the STR-Tree and Soft Tree models discussed in the Introduction handle categorical predictors natively, and all of them rely on transformations of categorical variables into numerical ones. Since missing values frequently occur in categorical features, this is a limitation of the present study, which is further discussed in Section~\ref{sec:conc}. Table~\ref{tab:datasets} summarizes the main characteristics of the datasets considered in the empirical study.

\begin{table}[ht]
\caption{Summary of the datasets used in the empirical evaluation: $n$ and $p$ denote the number of observations and the number of numerical predictors, respectively; $m_{global}$ is the percentage of missing predictor cells among all $n\times p$ predictor values, that is, $m_{global} = 100 \times (\text{number of missing predictor cells})/(np)$; $n_{miss}$ is the number of observations containing at least one missing predictor value, and the value in parentheses is the corresponding percentage of the $n$ observations.}\label{tab:datasets}
\centering
\begin{tabular}[t]{lcccrrrr}
\toprule
Dataset & Source & UCI ID & Response range & $n$ & $p$ & $m_{global}(\%)$ & $n_{miss}$ (\%)\\
\midrule
Crime & UCI & 183 & $[0, 1]$ & 1994 & 122 & 15.15 &1675 (84.00)\\
Garment & UCI & 597 & $[0.23, 1.12]$ & 1197 & 10 & 4.23 & 506 (42.27)\\
Auto MPG & UCI & 9 & $[9, 46.6]$ & 398 & 7 & 0.22 &6\,\, (1.51)\\
Ozone & mlbench & - & $[1, 38]$ & 361 & 9 & 6.03 &158 (43.77)\\
Automobile & UCI & 10 & $[-2, 3]$ & 205 & 17 & 1.69 &46 (22.44)\\
\bottomrule
\end{tabular}
\end{table}

The datasets differ substantially not only in their overall proportion of missing predictor values, but also in the fraction of observations affected by missingness. For example, although the Ozone dataset contains only $6.03\%$ missing predictor cells, approximately $44\%$ of its observations contain at least one missing predictor value. A similar phenomenon occurs in the Crime dataset, where $15.15\%$ of the predictor cells are missing but nearly $84\%$ of the observations are affected. By contrast, only six observations in Auto MPG contain missing predictor values. These differences are expected to influence the potential benefits of the probabilistic mechanisms introduced in PRTree, since they determine how frequently incomplete observations must be propagated through the tree.

\subsection{Experimental Configuration}\label{subsec:config}

\paragraph{Shared configuration.}
Predictive performance was assessed using the Root Mean Squared Error (RMSE) under 15-fold cross-validation. The same data partitions were used across all competing methods and experimental configurations in order to ensure direct comparability between predictive results. In addition to predictive accuracy, computational time was recorded for all methods and configurations. Reported runtimes correspond to the complete model-fitting procedure associated with each method.

\paragraph{CART configuration.}
The \texttt{rpart} implementation was used with internal pruning cross-validation configured through \texttt{xval = 10}. The experimental configuration employed complexity parameter \texttt{cp = 0.001}, minimum splitting size supplied through \texttt{minsplit}, defined as approximately 1\% of the effective training sample size, and maximum tree depth supplied through \texttt{maxdepth = 30} (the maximum allowed by the package). CART control parameters were initialized from the default \texttt{rpart.control()} configuration and updated whenever compatible arguments were supplied through the experimental configuration. Tree pruning was subsequently performed using the complexity parameter associated with the minimum internal cross-validated prediction error in the corresponding cost-complexity table.

\paragraph{PRTree configuration.}
Tree construction was controlled through \texttt{cp = 0.001}, meaning that a split was accepted only if the corresponding reduction in MSE exceeded 0.1\%. The maximum number of terminal regions was fixed at \texttt{max\_terminal\_nodes = 100}, while \texttt{prop\_min = 0.01} yielded a minimum node size of approximately 1\% of the effective training sample size. Additional probability-based stopping restrictions were imposed through \texttt{perc\_x = 0.1} and \texttt{p\_min = 0.05}, corresponding to the default values of the \texttt{PRTree} package. Candidate split search employed \texttt{by\_node = TRUE} and $\texttt{n\_candidates} = \max(\lfloor p/2 \rfloor,3)$, where $p$ denotes the number of predictors in the corresponding dataset.

Smoothing-parameter selection employed \texttt{grid\_size = 8}, corresponding to the default value of the \texttt{PRTree} package and generating eight positive candidate values for each component of $\bs\sigma$. The candidate grid was augmented with the additional value \texttt{tiny\_sigma = 0}, allowing the limiting case of vanishing smoothing parameters to recover the hard-split behavior of classical CART trees. The parameter vector $\bs\sigma$ was selected using an internal validation sample approximately equal in size to a single cross-validation fold. For the smoothing density function $\phi$, nine specifications were considered: Gaussian, Student-$t$, lognormal, and Gamma. The Student-$t$ specification employed degrees of freedom $\texttt{df}\in\{3,5\}$. The lognormal specification employed dispersion parameters $\texttt{sdlog}\in\{0.5,1,2\}$. The Gamma specification employed shape parameters $\texttt{shape}\in\{0.5,2,5\}$.

The empirical evaluation considered the three proposed missing-data extensions associated with \texttt{fill\_type} values 0, 1, and 2, together with the three Stage 1 proxy-selection criteria \texttt{"mean"}, \texttt{"var"}, and \texttt{"both"}. Combining the three missing-data strategies, three proxy criteria, and nine smoothing density specifications resulted in 81 distinct PRTree configurations for each dataset.

\paragraph{Notation for configurations.} Throughout the paper, the labels \texttt{f0}, \texttt{f1}, and \texttt{f2} denote \texttt{fill\_type} = 0, 1, and 2, respectively. The letters \texttt{M}, \texttt{V}, and \texttt{B} denote the proxy-selection criteria \texttt{"mean"} (mean-based), \texttt{"var"} (variance-based), and \texttt{"both"} (combined). For the smoothing densities, \texttt{N} denotes the Gaussian density, \texttt{t(3)} and \texttt{t(5)} denote the Student-$t$ densities with 3 and 5 degrees of freedom, \texttt{LN} denotes the log-normal densities (with $\texttt{sdlog}\in\{0.5,1,2\}$), and \texttt{G} denotes the Gamma densities (with $\texttt{shape}\in\{0.5,2,5\}$).

\paragraph{Reproducibility.} The scripts used to prepare the data, generate the cross-validation partitions, fit all models, run the semi-synthetic study of Section~\ref{sec:sim} (including the two additional designs described there), and produce all figures and tables are available at github.\footnote{\url{https://github.com/tsprass/Codes_PRTree_paper}}

\subsection{Relative predictive error}
\label{subsec:relative-error}

The datasets considered in the empirical study correspond to distinct prediction problems and differ substantially in scale and predictive characteristics. Consequently, direct comparisons of raw RMSE values across datasets are of limited interest. For this reason, each PRTree configuration was evaluated relative to the corresponding CART performance obtained on the same dataset and cross-validation fold.

Let $\mathrm{RMSE}_{d}(m,k)$ denote the root mean squared error obtained on dataset $d$ by method $m$ on fold $k$, and let $\mathrm{RMSE}_{d}^{\texttt{CART}}(k)$ denote the RMSE obtained by CART on the same fold. Relative predictive error was defined as
\[
\Delta_{d}(m,k)=\frac{\mathrm{RMSE}_{d}(m,k)-\mathrm{RMSE}_{d}^{\texttt{CART}}(k)}{\mathrm{RMSE}_{d}^{\texttt{CART}}(k)}.
\]
Under this definition, $\Delta_{d}(m,k)=0$ indicates identical predictive performance, negative values indicate that the corresponding PRTree configuration outperformed CART, and positive values indicate inferior performance relative to CART.

The quantity $\Delta_{d}(m,k)$ measures predictive performance relative to a common baseline within each dataset. Although the datasets remain inherently different and should not be interpreted as directly comparable prediction tasks, the relative error makes it possible to assess whether particular modeling choices, such as the missing-data handling strategy, proxy-selection criterion, or smoothing specification, tend to improve or deteriorate predictive performance across the collection of datasets considered in this study. Moreover, because all methods and configurations were evaluated using identical cross-validation partitions within each dataset, relative predictive errors can be compared directly across competing PRTree configurations.

\subsection{Results}\label{subsec:results}

\subsubsection{Predictive performance relative to CART}

To better understand the contribution of the probabilistic components introduced in PRTree, Figures~\ref{fig:delta-boxplots} and~\ref{fig:winrate-heatmap} compare its predictive performance against the classical CART algorithm. Throughout this section, negative values of $\Delta$ indicate that PRTree achieved lower prediction error than CART, whereas positive values indicate the opposite.

\begin{figure}[!ht]
\centering
\includegraphics[width=\textwidth]{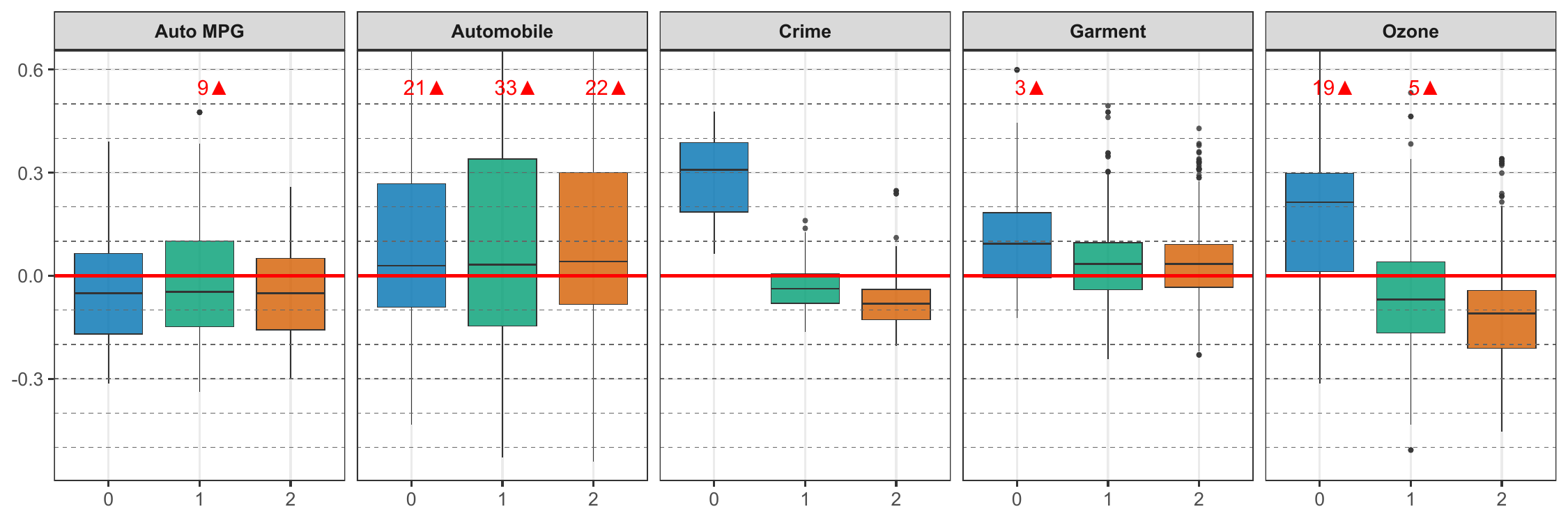}
\caption{Distribution of relative predictive errors grouped by \texttt{fill\_type}. Negative values indicate that PRTree outperformed CART. Labels \texttt{f0}, \texttt{f1}, \texttt{f2}: \texttt{fill\_type} = 0, 1, 2.}
\label{fig:delta-boxplots}
\end{figure}

Figure~\ref{fig:delta-boxplots} provides an overall summary of the effect of the fill strategy. For each dataset, the boxplots display the distribution of $\Delta$ values obtained across all smoothing distributions, proxy-selection criteria, and cross-validation folds associated with a given value of \texttt{fill\_type}. The horizontal red line corresponds to $\Delta=0$, representing equal predictive performance between PRTree and CART. To improve readability, the vertical axis was truncated at a fixed positive threshold. Whenever observations exceeded this limit, the symbol \texttt{n$\blacktriangle$} was added above the corresponding boxplot, where $n$ denotes the number of omitted observations. Therefore, these annotations identify configurations that produced larger prediction errors than those displayed within the plotting range.

While the boxplots summarize the overall behavior of each fill strategy, they do not distinguish between smoothing distributions or proxy-selection criteria. A more detailed view is provided by Figure~\ref{fig:winrate-heatmap}, which reports the empirical win rate of every evaluated configuration. In Figure~\ref{fig:winrate-heatmap}, rows correspond to combinations of fill strategy and proxy-selection criterion, and columns correspond to the smoothing distributions, with the notation introduced in Section~\ref{subsec:config}. Each cell contains the proportion of cross-validation folds in which PRTree achieved lower prediction error than CART. Values above $50\%$ indicate that PRTree outperformed CART more often than not, whereas values below $50\%$ indicate the opposite.

\begin{figure}[!ht]
\centering
\includegraphics[width=\textwidth]{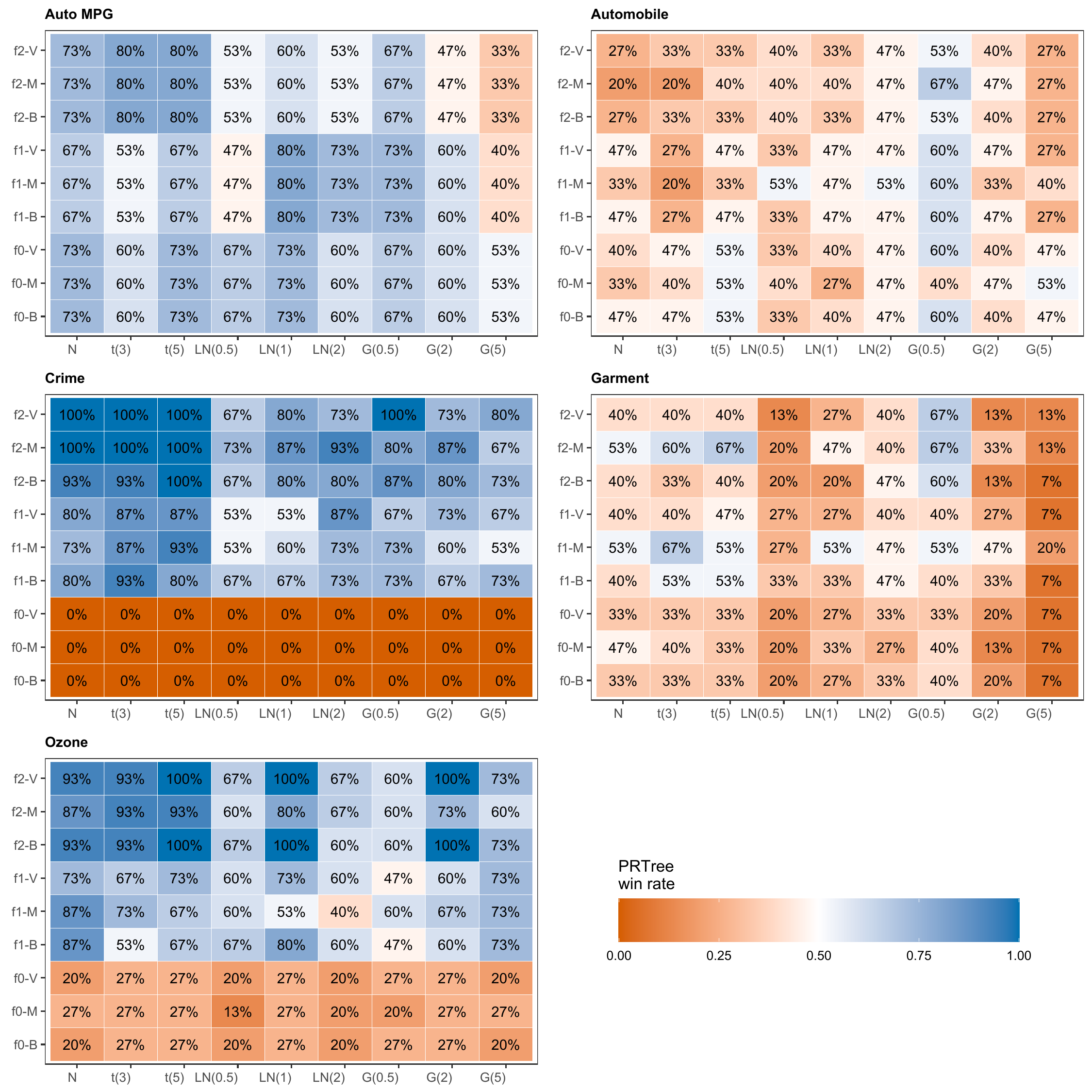}
\caption{Empirical win rate of each PRTree configuration relative to CART. Rows: \texttt{f0}--\texttt{f2} (\texttt{fill\_type} = 0--2) combined with the proxy criteria \texttt{B} (both), \texttt{M} (mean), \texttt{V} (variance). Columns: smoothing densities \texttt{N} (Gaussian), \texttt{t($\nu$)} (Student-$t$), \texttt{LN} (log-normal), \texttt{G} (Gamma), with their parameters.}
\label{fig:winrate-heatmap}
\end{figure}

The most important result emerging from Figures~\ref{fig:delta-boxplots} and~\ref{fig:winrate-heatmap} is the dominant role played by the fill strategy in the datasets with a large proportion of incomplete observations. In Crime and Ozone, changes in \texttt{fill\_type} produce substantially larger variations in predictive performance than changes in either the smoothing distribution or the proxy-selection criterion. In Auto MPG, Automobile and Garment, by contrast, the fill strategy has little influence on the win rates, and the smoothing distribution is the component that varies them most. Configurations sharing the same fill strategy in Crime and Ozone exhibit similar win rates even when different distributions and proxy criteria are employed, whereas changing the fill strategy can completely alter the relative performance of PRTree with respect to CART.

The influence of the proxy-selection criterion appears comparatively small. Within a fixed fill strategy, the rows associated with the mean-based, variance-based, and combined criteria usually display similar win rates, and no criterion consistently dominates across datasets. This suggests that the proxy criterion acts primarily as a secondary tuning parameter whose influence is weaker than that of the smoothing distribution and of the fill strategy.

The role of the smoothing distribution is more subtle. The heatmaps indicate that distributional choices can affect predictive performance, but their influence is highly dependent on both the selected fill strategy and the characteristics of the dataset. In datasets such as Crime and Ozone, where \texttt{fill\_type} = 1 and 2 clearly outperform \texttt{fill\_type} = 0, the improvement is observed across a broad range of distributions. In these cases, changing the distribution typically modifies the magnitude of the gain rather than altering the overall ranking of the configurations. This suggests that most of the predictive improvement is attributable to the fill mechanism itself, while the smoothing distribution acts primarily as a refinement component. In the remaining datasets, where few observations are incomplete or the missingness is structural, the smoothing distribution is instead the component that most affects the win rates.

Nevertheless, some distributional patterns can still be identified. Student-$t$ distributions frequently appear among the strongest configurations, particularly when combined with the most successful fill strategies. Although the experiments do not directly explain this behavior, the results suggest that heavier-tailed smoothing mechanisms may provide a more effective allocation of incomplete observations near splitting boundaries. At the same time, the superiority of Student-$t$ distributions is not universal, and competitive results are frequently obtained using log-normal and Gamma distributions. Consequently, the choice of smoothing distribution is a refinement once an appropriate fill strategy has been selected, but it rarely compensates for an inadequate fill mechanism when the fill strategy is decisive (Crime and Ozone); when few observations are incomplete, it can matter more than the fill strategy.

\paragraph{Auto MPG.}  The Auto MPG dataset illustrates a relatively stable scenario. The three fill strategies produce similar boxplots, all centered slightly below zero, indicating that PRTree tends to outperform CART regardless of the specific probabilistic configuration. The heatmap reinforces this conclusion, with most configurations achieving win rates above $50\%$ and several exceeding $70\%$. Although \texttt{fill\_type}=2 appears to provide the strongest overall performance, the differences among fill strategies remain moderate. One possible explanation is the extremely low prevalence of missing values in this dataset. As shown in Table~\ref{tab:datasets}, only six observations ($1.51\%$ of the sample) contain at least one missing predictor value, corresponding to merely $0.22\%$ of all predictor cells. Consequently, the probabilistic components of PRTree are activated only rarely, limiting the potential impact of the fill strategy on predictive performance. This is consistent with the relatively small differences observed among the competing configurations. It also suggests that, in this dataset, the advantage of PRTree over CART is unlikely to be due to the treatment of missing values and is more plausibly explained by the smooth probabilistic assignments of the base model.

\paragraph{Automobile.} The Automobile dataset presents the most balanced comparison among all considered datasets. The boxplots exhibit substantial variability and a considerable number of observations above the plotting threshold, indicating that some configurations can perform substantially worse than CART. Consistently, the heatmap shows that most win rates lie between $30\%$ and $60\%$, with only a few configurations exceeding these limits. No clear pattern emerges with respect to either fill strategy or smoothing distribution. This absence of a dominant structure suggests that neither method consistently outperforms the other in this setting. The characteristics of the dataset reported in Table~\ref{tab:datasets} are consistent with this behavior. The overall proportion of missing predictor cells is relatively small ($1.69\%$), and only $22.44\%$ of the observations contain at least one missing predictor value. Consequently, the probabilistic mechanisms affect only a limited portion of the data, reducing the opportunity for PRTree to obtain substantial gains over CART. Moreover, Automobile is the smallest dataset considered in the study ($n=205$), making performance estimates more sensitive to variations in the training sample. The combination of a limited amount of missing information and a relatively small sample size may therefore help explain the absence of a clear advantage for either method.

\paragraph{Communities and Crime.} The Crime dataset provides the strongest evidence of the importance of the fill mechanism. The contrast between \texttt{fill\_type} = 0 and the remaining strategies is striking. Every configuration associated with \texttt{fill\_type} = 0 loses to CART in all cross-validation folds, producing win rates equal to $0\%$. However, once the fill mechanism is modified, the behavior changes completely. Most configurations associated with \texttt{fill\_type} = 1 and 2 achieve win rates above $80\%$, and several reach $100\%$. Such a dramatic transition cannot be attributed to the choice of distribution or proxy criterion, since it occurs consistently across almost all configurations within each fill strategy. Instead, it strongly suggests that the way incomplete observations are propagated through the tree is the dominant factor determining predictive performance in this dataset. The characteristics of the dataset reported in Table~\ref{tab:datasets} help explain this behavior. The Crime dataset exhibits both the largest proportion of missing predictor cells ($15.15\%$) and the largest proportion of affected observations ($84.00\%$). As a result, the treatment of incomplete observations becomes a central component of the learning procedure. Under these conditions, differences in the fill mechanism have a direct impact on a substantial fraction of the training data, which helps explain the dramatic performance gap observed between \texttt{fill\_type} = 0 and the remaining strategies.

\paragraph{Productivity Prediction of Garment Employees.} The Garment dataset exhibits the opposite behavior of Crime and Ozone. Most configurations achieve win rates below $50\%$, indicating that CART generally provides lower prediction errors. The heatmap reveals that this pattern is remarkably consistent across distributions and proxy-selection criteria, suggesting that the observed disadvantage is not driven by a particular probabilistic choice. Unlike Crime and Ozone, where the probabilistic treatment of missing observations produces substantial gains, the Garment dataset provides little evidence that the additional flexibility introduced by PRTree translates into systematic improvements in predictive performance. Although a few isolated configurations achieve competitive results, CART remains the most reliable choice overall. Interestingly, this result shows that a large number of incomplete observations does not automatically imply that probabilistic treatment will improve predictive performance. Although more than $42\%$ of the observations contain missing predictor values, CART remains competitive and often superior. This suggests that the effectiveness of the probabilistic mechanisms depends not only on the amount of missing information, but also on how that missingness interacts with the underlying structure of the prediction problem.

The missingness pattern in Garment is structural. Each of the 506 incomplete observations has exactly one missing value, and it is always the work in progress (\texttt{wip}): \texttt{wip} is missing for every observation of the finishing department and observed for every observation of the sewing department. The department, which was excluded from the analysis, therefore determines the missingness completely, so the data are not MCAR and the missingness indicator is informative. This is a difficult situation for PRTree. For an observation with missing \texttt{wip}, every split on \texttt{wip} sends half of its association to each child (Remark~\ref{rmk:mechanism}), so its prediction averages the leaf coefficients on both sides of the split. In a single PRTree fitted to the full sample (\texttt{fill\_type} = 2, Gaussian density, combined proxy criterion, settings of Section~\ref{subsec:config}), \texttt{wip} defines the root split and is used in six of the 18 internal nodes. Hence every finishing-department observation is blurred at the root and again at each later split on \texttt{wip}. In a single CART fitted to the same sample, by contrast, \texttt{wip} is among the least important variables (surrogates included), and the tree is driven by \texttt{targeted\_productivity}, \texttt{smv} and \texttt{no\_of\_workers}, the last two possibly acting as proxies for the department. CART thus avoids the incomplete predictor, while PRTree builds its tree on it from the root.

To check this explanation, we refitted PRTree on Garment under controlled conditions (Gaussian density, same folds and tree-growing settings as in Section~\ref{subsec:config}; code and results are available in the GitHub repository). Without \texttt{wip}, that is, with all predictors observed, PRTree and CART are practically tied (mean RMSE $0.1325$ against $0.1345$ for CART). With \texttt{wip}, PRTree without smoothing ($\sigma=0$) is $23\%$ worse than CART (RMSE $0.164$ against $0.134$), because every split on \texttt{wip} sends the incomplete observations half to each child. The smoothing selected on the validation sample removes most of this loss, but with the default splitting constraint \texttt{prop\_x} $=0.1$, defined together with \texttt{p\_min} in Section~\ref{subsec:config}, PRTree remains above CART (RMSE $0.1383$ against $0.1338$). When this constraint is removed (\texttt{prop\_x} $=\texttt{p\_min}=0$), the difference vanishes ($0.1333$ against $0.1338$), and enlarging the $\sigma$ grid up to four standard deviations of each predictor gives $0.1318$. All these differences with \texttt{prop\_x} $=0$ are of the order of the standard error of the mean difference over the 15 folds ($0.003$ to $0.004$), so PRTree is competitive with CART on Garment, but not superior. The disadvantage under the default settings is therefore associated with the combination of the structurally missing predictor \texttt{wip} and the splitting constraint, and not with the proportion of incomplete observations. Moreover, smoothing the continuous predictors only, or the discrete ones only, does not improve the results: fixing $\sigma=0$ for the predictors with at most 12 distinct values leaves the RMSE essentially unchanged ($0.1336$), so we found no evidence that the discrete predictors explain the result.

\paragraph{Ozone.} The Ozone dataset displays a pattern similar to that observed for Crime, although less extreme. A clear monotonic improvement is visible when moving from \texttt{fill\_type} = 0 to \texttt{fill\_type} = 2. The boxplots progressively shift from predominantly positive relative errors to predominantly negative ones, while the heatmap reveals a corresponding increase in win rates. Several configurations associated with \texttt{fill\_type} = 2 achieve win rates above $90\%$, whereas most configurations associated with \texttt{fill\_type} = 0 remain below $30\%$. The consistency of this pattern across smoothing distributions and proxy-selection criteria reinforces the conclusion that the fill mechanism is the primary driver of predictive performance. The dataset characteristics reported in Table~\ref{tab:datasets} provide a plausible explanation for this behavior. Although only $6.03\%$ of the predictor cells are missing, approximately $44\%$ of the observations contain at least one missing predictor value. Consequently, a large fraction of the sample is affected by missingness, making the treatment of incomplete observations an important component of the learning process. This is consistent with the strong improvements observed when moving from \texttt{fill\_type} = 0 to the more sophisticated fill strategies.

\paragraph{Summary.} Taken together, the results show that the relative importance of the design choices depends on how much of the data is incomplete. In Crime and Ozone, where a large share of the observations has missing predictors, the fill strategy is by far the most influential component, frequently determining whether PRTree outperforms or underperforms CART, and the smoothing distribution acts as a refinement. In Auto MPG, Automobile and Garment, where missing values are rare or structurally determined, the fill strategy has little influence on the win rates, and the smoothing distribution is the component that varies them most. The proxy-selection criterion generally exerts only a minor influence in all datasets. From a practical perspective, the fill strategy should be chosen first when a considerable fraction of the observations is incomplete, whereas the smoothing distribution deserves more attention when few observations are incomplete.

\subsubsection{Computational cost}

In addition to predictive accuracy, computational cost was recorded for all methods and configurations considered in the empirical study. The reported runtimes were obtained on a desktop computer equipped with an Intel\textsuperscript{\textregistered} Core\textsuperscript{\texttrademark}
%Intel(R) Core(TM)
i5-4590 CPU (3.30\,GHz) and 16\,GB of RAM, running Windows 10, which is a modest configuration by current standards (a four-core desktop processor released in 2014). This choice was deliberate: the aim is to give an ordinary user a realistic idea of the time required to run the analysis, and since most users are expected to have a more powerful computer, shorter runtimes should be expected in practice. The datasets were processed in parallel, one per core. The absolute runtimes should therefore be interpreted as conservative, and the relevant information is the relative cost of the different configurations and datasets. Table~\ref{tab:runtime} summarizes the observed runtimes for CART and PRTree across all datasets. Along with the computational results, the table reports the sample size $n$, predictor dimension $p$, and the number of candidate splits fully evaluated at each node, $n_{\mathrm{cand}}=\max\{p/2,3\}$. These quantities are useful for interpreting the observed runtimes, since the computational burden of PRTree depends both on the amount of data processed at each node and on the number of candidate splits receiving a full MSE evaluation.

\begin{table}[ht]
\caption{Computational cost of CART and PRTree. For each dataset, the table reports the sample size $n$, predictor dimension $p$, the number of candidate splits fully evaluated at each node ($n_{\mathrm{cand}}$), the total runtime of CART (in seconds), the cumulative runtime required to evaluate all 81 PRTree configurations (in minutes), and the minimum and maximum runtimes observed among individual PRTree configurations (in seconds).}
\label{tab:runtime}
\centering
\begin{tabular}{lrrrrrrr}
\toprule
Dataset & $n$ & $p$ & $n_{\mathrm{cand}}$ & CART (s) & PRTree (min) & PRTree min (s) & PRTree max (s)\\
\midrule
Crime      & 1994 & 122 & 61 & 8 & 2106.05 & 308 & 7966\\
Garment    & 1197 & 10  & 5  & 1 & 100.60  & 22  & 217\\
Auto MPG   & 398  & 7   & 4  & 1 & 56.23   & 5   & 112\\
Ozone      & 361  & 9   & 4  & 1 & 55.45   & 10  & 118\\
Automobile & 205  & 17  & 8  & 1 & 63.10   & 9   & 116\\
\bottomrule
\end{tabular}
\end{table}

The computational burden of PRTree is substantially larger than that of CART. This increase is expected because each PRTree fit requires repeated evaluations of probabilistic node assignments and an internal smoothing-parameter selection procedure. For every configuration, nine candidate values of the smoothing parameter are evaluated, corresponding to the eight values generated by \texttt{grid\_size = 8} together with the additional value \texttt{tiny\_sigma = 0}. Since the tree structure depends on the smoothing parameter, each candidate value requires a complete reconstruction of the tree.

An additional source of computational cost arises from the split-selection procedure. At each node, all possible splits are first ranked using the selected proxy criterion. The best $n_{\mathrm{cand}}=\max\{p/2,3\}$ candidate splits are then fully evaluated using the MSE criterion. Consequently, the computational effort associated with split selection depends both on the number of observations contributing to the node statistics and on the number of candidate splits receiving a full MSE evaluation.

The runtimes reported in Table~\ref{tab:runtime} suggest that both sample size and $n_{\mathrm{cand}}$ play an important role in determining computational cost. The Crime dataset required more than 35 hours to evaluate all configurations, whereas the remaining datasets required between approximately one and two hours. This behavior is consistent with the fact that Crime combines the largest sample size ($n=1994$) with the largest value of $n_{\mathrm{cand}}$ ($61$). By contrast, although Automobile evaluates more candidate splits than Garment ($8$ versus $5$), its substantially smaller sample size ($205$ versus $1197$ observations) results in a lower overall runtime. Taken together, these results suggest that the computational burden of PRTree is driven by the interaction between the amount of data processed and the number of candidate splits receiving a full MSE evaluation.

Missingness may also affect computational cost because incomplete observations require additional processing during the calculation of the proxy measures used to rank candidate splits. As the proportion of missing values increases, these calculations become more frequent throughout tree construction. Nevertheless, the runtimes reported in Table~\ref{tab:runtime} do not suggest that missingness alone is the dominant driver of computational cost. The largest differences in runtime are more closely associated with the combined effects of sample size and the number of candidate splits receiving a full MSE evaluation.

To investigate how the probabilistic design choices affect computational cost, Figure~\ref{fig:time-heatmap} reports the average runtime of each evaluated configuration. The notation follows the same convention adopted in Figure~\ref{fig:winrate-heatmap}: rows correspond to combinations of fill strategy and proxy-selection criterion, while columns correspond to the smoothing distributions and their associated parameters.

\begin{figure}[!ht]
\centering
\includegraphics[width=\textwidth]{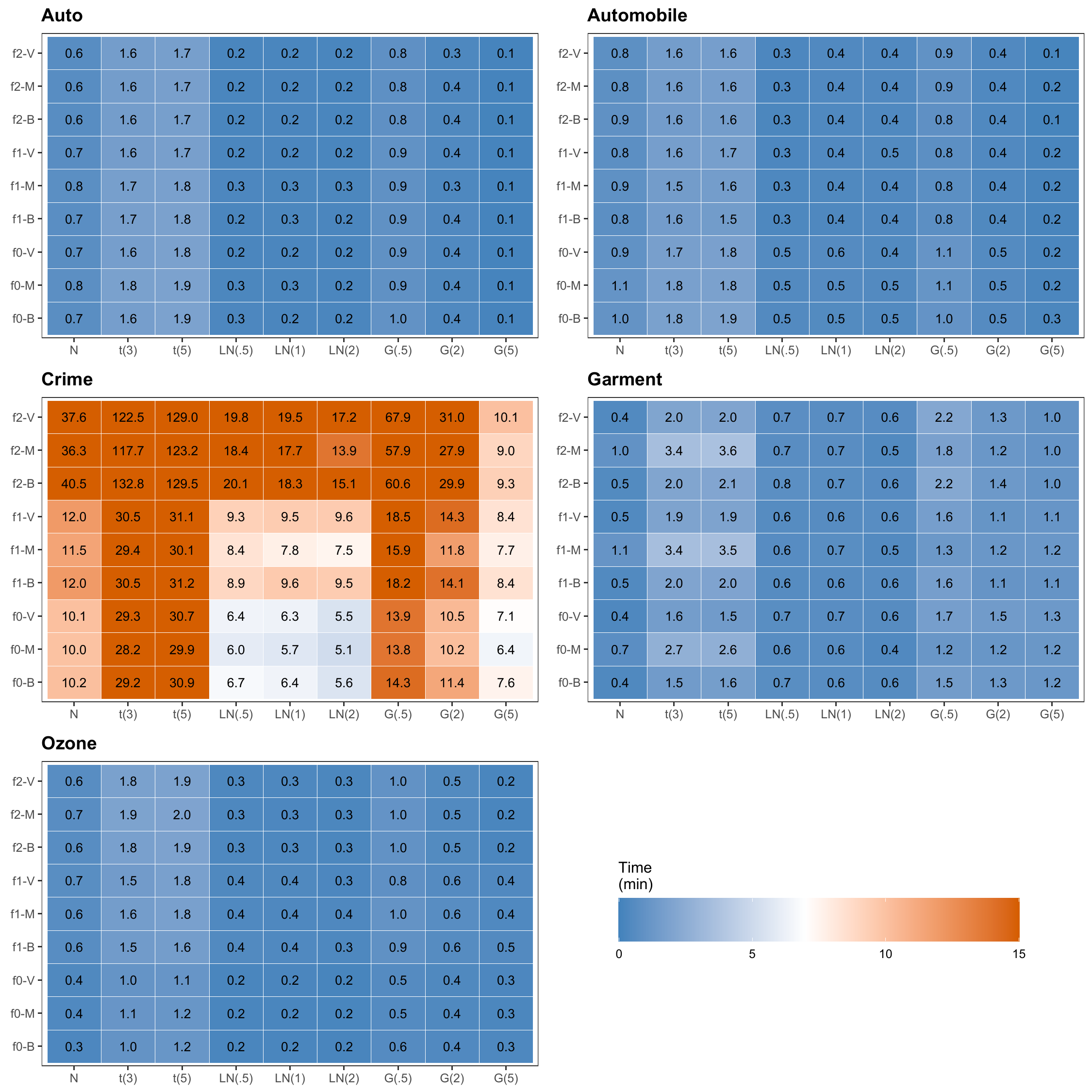}
\caption{Average runtime (minutes) of each PRTree configuration. Rows and columns as in Figure~\ref{fig:winrate-heatmap}.}
\label{fig:time-heatmap}
\end{figure}

Several patterns emerge from Figure~\ref{fig:time-heatmap}. First, the choice of smoothing distribution has a substantial impact on runtime. Across all datasets, the Student-$t$ distributions are consistently the most computationally demanding alternatives. The configurations associated with \texttt{t(3)} and \texttt{t(5)} systematically require longer execution times than the corresponding Gaussian, log-normal, and Gamma configurations. In contrast, the log-normal distributions are typically among the fastest options, while the Gamma distributions occupy an intermediate position.

Second, the fill strategy also affects computational cost. Within a given dataset, runtimes generally increase when moving from \texttt{fill\_type} = 0 to \texttt{fill\_type} = 2. This pattern is particularly visible for the Crime dataset, where some \texttt{fill\_type} = 2 configurations require several times more computation than their \texttt{fill\_type} = 0 counterparts. Similar, although less pronounced, patterns can also be observed in the remaining datasets. Unlike the large differences in runtime observed across datasets, which are largely associated with predictor dimension and the corresponding value of $n_{\mathrm{cand}}$, these within-dataset variations are likely related to the additional probabilistic allocations required by more sophisticated fill mechanisms.

Comparing Figures~\ref{fig:winrate-heatmap} and~\ref{fig:time-heatmap} reveals that the most computationally expensive configurations are not necessarily those with the strongest predictive performance. In the Crime and Ozone datasets, many of the best-performing configurations are indeed associated with \texttt{fill\_type} = 2, which is also the most expensive fill strategy. However, within a fixed fill strategy, the distributions that require the longest runtimes do not systematically achieve the highest win rates. Student-$t$ distributions are often among the slowest configurations, yet Gamma and log-normal distributions frequently attain comparable predictive performance at substantially lower computational cost.

Taken together, the results suggest that different factors govern runtime at different levels of the analysis. Across datasets, computational cost is driven primarily by predictor dimension through its effect on the number of candidate splits that receive a full MSE evaluation. Within a given dataset, runtime is further influenced by the fill strategy and, to a lesser extent, by the choice of smoothing distribution. Although some of the most accurate configurations are also among the most computationally demanding, the relationship between runtime and predictive performance is far from monotonic. Consequently, computational cost should be viewed as a consequence of specific algorithmic design choices rather than as a direct indicator of predictive quality.

\paragraph{Implementation and computational strategies.} The focus of this paper is the methodology, and the reported runtimes refer to the implementation in the \texttt{PRTree} package, whose computational design is described in detail in the companion paper \citep{PRTree2026}. In brief, the implementation already incorporates several strategies commonly adopted in tree-based methods to reduce cost: the computationally intensive components are written in \texttt{FORTRAN} and \texttt{C}; the association density is factorized, so that only one-dimensional marginal probabilities of the splitting variable are computed at each split (see the discussion after Lemma~\ref{lem:product}); candidate splits are first screened by an inexpensive proxy score, which avoids rebuilding the association matrix $P$ and re-estimating $\bs\gamma$ for every threshold, and only the $n_{\mathrm{cand}}$ best candidates receive a full MSE evaluation; candidate lists are computed once per node and reused until the node is split; and nodes, as well as thresholds, that cannot produce two children satisfying the probability constraints are discarded before evaluation, which restricts the regions that can be split. The additional cost of \texttt{fill\_type} = 2 arises because, for each incomplete observation, $\Psi$ must be evaluated on the projected region rather than replaced by a constant (\texttt{fill\_type} = 0) or an indicator (\texttt{fill\_type} = 1). Finally, since each candidate value of $\bs\sigma$ requires an independent reconstruction of the tree, the evaluation of the $\bs\sigma$ grid can be distributed by the user across parallel processes, for instance by supplying a single candidate through \texttt{sigma\_grid} in each process. Parallelism is deliberately not implemented inside the compiled code, to avoid conflicts with parallel environments set up by the user.

\section{Comparison with imputation-based pipelines under controlled \linebreak missingness}\label{sec:sim}

The study in Section~\ref{sec:aplic} has two limitations. First, the missingness mechanism of the real datasets is unknown, so that its effect on the proposed strategies cannot be isolated. Second, CART with surrogate splits is the only competitor, so that the gains of PRTree may be due either to the smoothness of the base model or to the internal treatment of missing values (as suggested by the Auto MPG results). To address both points, this section reports a semi-synthetic study in which missing values are injected, under controlled MCAR, MAR and MNAR mechanisms, into datasets with fully observed predictors, and the proposed strategies are compared with impute-then-fit pipelines. In line with the scope of the paper, the study is deliberately focused: it considers a small number of datasets and a single PRTree specification, rather than the full set of 81 configurations.

\subsection{Design}

\paragraph{Datasets.} Three datasets with fully observed predictors are considered: the complete cases of Auto MPG ($n = 392$, $p = 7$) and of Ozone ($n = 203$, $p = 9$), and the \textit{Concrete Compressive Strength} dataset (UCI ID 165; $n = 1030$, $p = 8$), which has no missing values. The complete-data fits provide a reference for the loss caused by the missing values.

\paragraph{Predictors with missing values.} Missing values are injected only into the $\lceil p/2 \rceil$ predictors most correlated with the response, in absolute value, computed once on the complete data (set $A$); the remaining $\lfloor p/2 \rfloor$ predictors (set $O$) are always observed. Predictors that are weakly related to the response may not even be used by the trees, so that deleting their values would hardly distinguish the procedures. The response variable dictates only the partition into $A$ and $O$, playing no role in determining which specific cells are deleted.

\paragraph{Missingness mechanisms.} Three mechanisms are considered, each calibrated so that the expected proportion of incomplete observations, that is, of observations with at least one missing predictor value, equals $r \in \{10\%, 25\%, 40\%\}$. Let $Z_{ij}$ denote the standardized value of $X_{ij}$.
\begin{itemize}
\item \textbf{MCAR.} Each cell of the predictors in $A$ is deleted independently with the same probability $q = 1-(1-r)^{1/|A|}$, so that the missingness pattern is independent of the data. This mimics values lost by random recording or transmission failures.
\item \textbf{MAR.} The $rn$ observations (rounded) with the largest value of $\max_{k \in O} \{Z_{ik}\}$ are flagged, with ties broken at random, and in each flagged observation every predictor in $A$ is deleted independently with probability $0.5$, the draw being repeated if no value is deleted. Hence the probability of missingness depends only on the values of the always-observed predictors, and the proportion of incomplete observations is $r$. This mimics a measuring device that tends to fail when another variable, which is always recorded, takes high values.
\item \textbf{MNAR.} In each predictor of $A$, the highest values are deleted, in the same number (up to one unit) in all predictors, chosen so that the proportion of incomplete observations equals $r$; ties are broken at random, so that discrete predictors are handled. This mimics a measuring device that saturates above a ceiling and records nothing, so that the probability of missingness depends on the deleted value itself.
\end{itemize}
Since the proportion of incomplete observations is fixed at $r$, the proportion of missing cells differs across mechanisms and datasets, because the mechanisms differ in the number of values deleted per incomplete observation. The realized proportions of incomplete observations are reported in Table~\ref{tab:sim} (column Inc.).

\paragraph{Other designs.} Two other designs were run with $R = 10$ replications, and their code is available in the repository (Section~\ref{sec:aplic}). In design A, any of the $p$ predictors can lose values: MCAR deletes cells independently; for MAR, $\lfloor p/2 \rfloor$ predictors drawn at random in each replication are always observed (set $O$), the observations with the largest $\max_{k \in O}\{|Z_{ik}|\}$ (two-sided) are flagged, and each of the remaining $\lceil p/2 \rceil$ predictors is deleted in the flagged observations with probability $0.5$ (at least one per observation); and MNAR deletes the lowest and the highest values (half of each) of every predictor. In design B, the predictors that can lose values are the set $A$ used here, MCAR is as above, MAR flags the observations with the largest $\max_{k \in O}\{|Z_{ik}|\}$, and MNAR deletes the lowest and the highest values of each predictor in $A$. In both designs, \texttt{PR-f2} is more accurate than CART with surrogate splits ($62\%$ to $96\%$ of the folds in design A and $70\%$ to $96\%$ in design B), than \texttt{PR-f1} ($52\%$ to $92\%$ and $58\%$ to $92\%$) and than \texttt{PR-f0} ($46\%$ to $100\%$ in both), but it is generally outperformed by impute-then-PRTree pipelines under all three mechanisms, including MNAR: it wins in $4\%$ to $60\%$ of the folds against \texttt{kNN+PR}, in $4\%$ to $68\%$ against \texttt{missForest+PR} and in at most $34\%$ against \texttt{MICE+PR}, and its median loss ranges from $-0.01$ to $0.39$ in design A and from $0.00$ to $0.56$ in design B. Hence the advantage of \texttt{PR-f2} under MNAR found in the main design depends on the direction of the deletion: it appears when the missing values are concentrated in one tail, which imputation fills with values typical of the observed ones, and not when both tails are removed.

\paragraph{Competing procedures.} The following procedures are compared:
\begin{enumerate}[(i)]
\item PRTree with the internal strategies \texttt{fill\_type} = 0, 1 and 2 (\texttt{PR-f0}, \texttt{PR-f1}, \texttt{PR-f2});
\item CART with surrogate splits (\texttt{CART-sur}), as in Section~\ref{sec:aplic};
\item impute-then-fit pipelines, in which the predictors are first imputed and then the original PRTree (\texttt{PR}) or CART is fitted to the completed data, with four imputation methods: unconditional mean imputation (\texttt{mean}); $k$-nearest-neighbor imputation with $k = 5$ \citep{troyanskaya2001} (\texttt{kNN}), in our own implementation: the predictors are standardized with the training means and standard deviations, the distance between two observations is the root mean squared difference over the predictors observed in both, and each missing value is replaced by the simple average of the values of that predictor in the $k$ nearest training observations in which it is observed, whereas \cite{troyanskaya2001} weight the neighbors by their similarity to the incomplete observation; random-forest imputation \citep{stekhoven2012} (\texttt{missForest}), in which each incomplete predictor is imputed by a random forest of CART-type regression trees fitted to the other predictors, iterating until convergence; and multiple imputation by chained equations with predictive mean matching \citep{vanbuuren2011} (\texttt{MICE}), with $m = 5$ imputed datasets, one model fitted to each of them, and predictions averaged over the $m$ models, which is the usual way of combining multiple imputations for prediction. Since averaging the predictions of several trees also reduces their variance, two single-tree versions of MICE are also considered, which use the same imputations: a tree fitted to the first imputed dataset (\texttt{MICE$_1$}), and a tree fitted to the cell-by-cell average of the $m$ imputed datasets (\texttt{MICE$_{\mathrm{avg}}$});
\item as a reference, PRTree and CART fitted to the original complete data (\texttt{PR-full}, \texttt{CART-full}).
\end{enumerate}
To avoid information leakage, imputation models are fitted on the training fold only, and the response is never used to impute the predictors, since it is not available for the test observations. For \texttt{mean} and \texttt{kNN}, the training means and the training observations (as donors) are used to impute both the training and the test folds; for \texttt{MICE}, the imputation models are estimated on the training rows only and then used to impute the test rows; for \texttt{missForest}, which does not provide a separate prediction step, the algorithm is applied to the training predictors, and the test predictors are then imputed by random forests fitted to the imputed training data, one for each incomplete predictor, iterating from the training means.

\paragraph{PRTree specification and protocol.} All PRTree fits use the Gaussian smoothing density and the combined proxy criterion (\texttt{proxy\_crit = "both"}), the default values of the \texttt{PRTree} package, and the remaining settings of Section~\ref{subsec:config}, including the selection of $\bs\sigma$ on an internal validation sample. CART is configured as in Section~\ref{subsec:config}. For both methods, the minimum node size is $1\%$ of the training sample, with a minimum of 3 observations. For each dataset, mechanism and rate, $R = 100$ replications are performed. Each replication generates a new missingness pattern and a new partition into $K = 5$ folds, which is shared by all procedures. Performance is assessed fold by fold through the loss with respect to PRTree fitted to the complete data, $L = \mathrm{RMSE}/\mathrm{RMSE}_{\texttt{PR-full}} - 1$, which is on the same scale as the relative error $\Delta$ of Section~\ref{subsec:relative-error} ($L = 0$ indicates no loss), and through the win rate of \texttt{PR-f2} against each competing procedure, that is, the percentage of folds in which \texttt{PR-f2} has a smaller RMSE. The computational cost is measured by the elapsed time of the imputation and of the fitting and prediction steps in each fold. The study was run in R on a MacBook Pro with an Apple M4 Max chip (14 cores), 36 GB of RAM and macOS Sequoia 15.7.7, with 13 tasks running in parallel. Uncertainty over replications is quantified by a percentile bootstrap: the $R = 100$ replications, each with its five folds, are resampled with replacement $B = 2000$ times, which preserves the dependence among the folds of a replication and among the procedures evaluated on the same folds; the 2.5\% and 97.5\% quantiles of the recomputed statistic give the 95\% confidence interval.

\subsection{Results}

Figure~\ref{fig:sim-heat} reports the win rate of \texttt{PR-f2} against each competing procedure, that is, the percentage of folds in which \texttt{PR-f2} has a smaller RMSE, for every dataset, mechanism and rate. Table~\ref{tab:sim} and Figure~\ref{fig:sim-box} summarize the loss $L$ of the six procedures with the highest overall win rate, computed as the proportion of pairwise comparisons (same fold) won by each procedure over all scenarios: \texttt{MICE+PR} ($86\%$), \texttt{kNN+PR} ($71\%$), \texttt{MICE$_{\mathrm{avg}}$+PR} ($69\%$), \texttt{PR-f2} ($68\%$), \texttt{missForest+PR} ($67\%$) and \texttt{MICE$_1$+PR} ($59\%$). Before discussing the missing-data strategies, we note that, with complete data, PRTree is clearly more accurate than CART: the average RMSE of \texttt{PR-full} is $2.88$ against $3.60$ for \texttt{CART-full} in Auto MPG, $4.44$ against $5.35$ in Ozone, and $6.66$ against $7.26$ in Concrete, and \texttt{PR-full} has the smaller RMSE in $96\%$, $90\%$ and $84\%$ of the folds, respectively. This confirms that part of the advantage of PRTree over CART observed in Section~\ref{sec:aplic} is due to the base model and not to the treatment of missing values.

\begin{figure}[!ht]
\centering
\includegraphics[width=\textwidth]{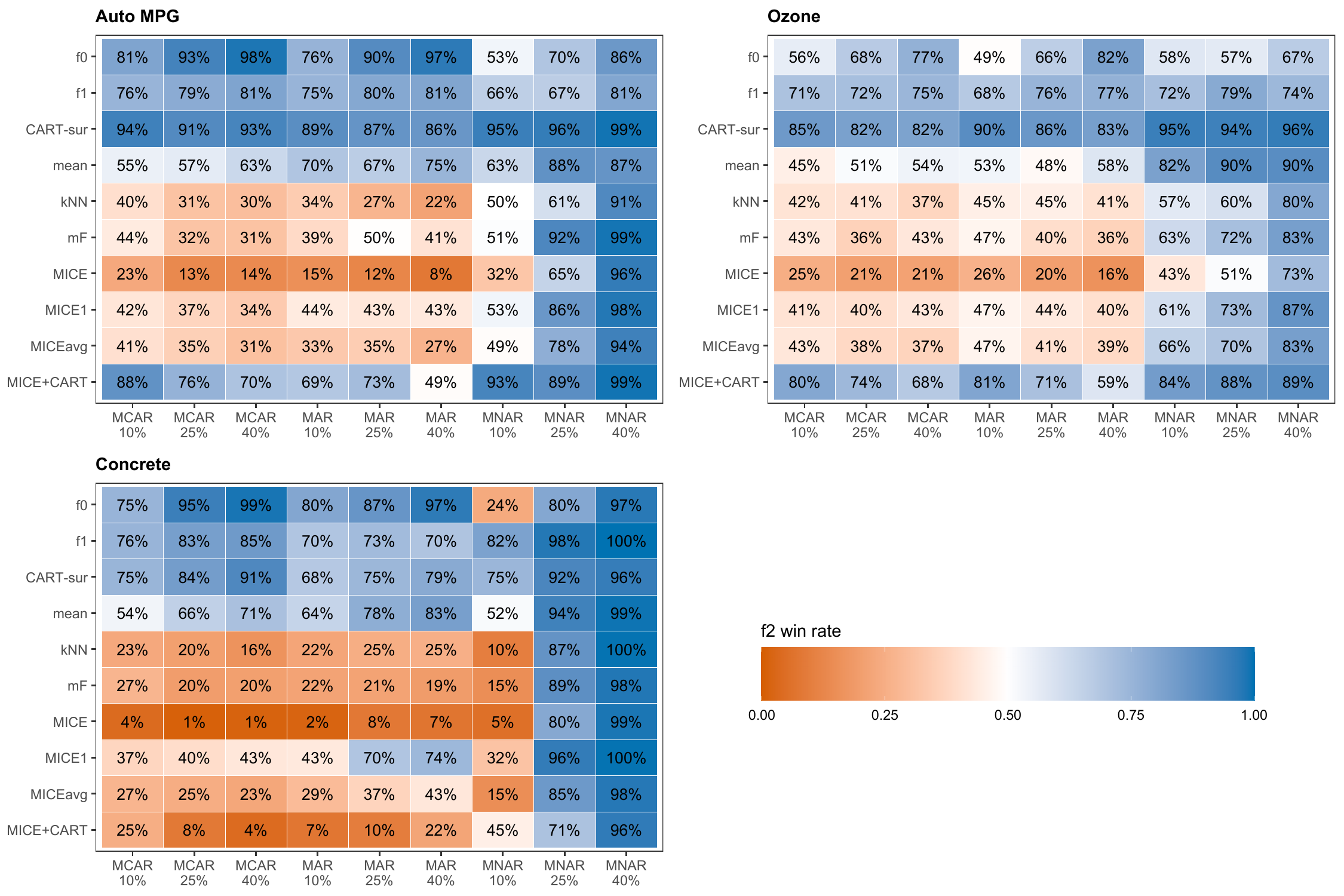}
\caption{Semi-synthetic study: win rate of \texttt{PR-f2} against each competing procedure (rows), that is, the percentage of folds in which \texttt{PR-f2} has a smaller RMSE, by mechanism and rate $r$ (columns). Competitors: PRTree with \texttt{fill\_type} = 0 and 1 (\texttt{f0}, \texttt{f1}); CART with surrogate splits (\texttt{CART-sur}); PRTree after mean, $k$NN, missForest (\texttt{mF}) and MICE imputation, where \texttt{MICE} averages $m = 5$ trees and \texttt{MICE1} and \texttt{MICEavg} are single trees fitted to the first imputed dataset and to the average of the imputed datasets; and CART after MICE imputation (\texttt{MICE+CART}).}\label{fig:sim-heat}
\end{figure}

{\scriptsize\setlength{\tabcolsep}{4pt}\renewcommand{\arraystretch}{1.05}
\begin{longtable}{llrcccccc}
\caption{Semi-synthetic study: median (interquartile range) over replications and folds of the loss $L = \mathrm{RMSE}/\mathrm{RMSE}_{\texttt{PR-full}} - 1$ for the six procedures with the highest overall win rate. $r$: target percentage of incomplete observations; Inc.: realized percentage of incomplete observations. MICE: predictions averaged over $m = 5$ trees; MICE$_{\mathrm{avg}}$: a single tree fitted to the average of the imputed datasets; MICE$_1$: a single tree fitted to the first imputed dataset. Negative values indicate a smaller RMSE than PRTree fitted to the complete data. In bold, the smallest median in each row. Below each median (interquartile range): 95\% percentile bootstrap confidence interval for the median, obtained by resampling the $R = 100$ replications. Medians, interquartile ranges and confidence limits are printed with three decimals.}\label{tab:sim}\\
\toprule
 & & & \multicolumn{6}{c}{PRTree}\\
\cmidrule(lr){4-9}
Mech. & $r$ & Inc. & MICE & $k$NN & MICE$_{\mathrm{avg}}$ & \texttt{f2} & mF & MICE$_1$\\
\midrule
\endfirsthead
\multicolumn{9}{l}{\textit{Table~\ref{tab:sim} (continued)}}\\
\toprule
 & & & \multicolumn{6}{c}{PRTree}\\
\cmidrule(lr){4-9}
Mech. & $r$ & Inc. & MICE & $k$NN & MICE$_{\mathrm{avg}}$ & \texttt{f2} & mF & MICE$_1$\\
\midrule
\endhead
\midrule
\multicolumn{9}{r}{\textit{continued on next page}}\\
\endfoot
\bottomrule
\endlastfoot
\multicolumn{9}{l}{\textit{Auto MPG}}\\*
MCAR & 10 & 10 & \best{-0.038 (0.068)} & 0.006 (0.080) & 0.001 (0.087) & 0.021 (0.129) & 0.008 (0.086) & 0.002 (0.078)\\*
 &  &  & \ic{[-0.044, -0.032]} & \ic{[-0.002, 0.011]} & \ic{[-0.005, 0.008]} & \ic{[0.013, 0.035]} & \ic{[0.002, 0.013]} & \ic{[-0.005, 0.007]}\\
 & 25 & 25 & \best{-0.028 (0.091)} & 0.014 (0.102) & 0.019 (0.104) & 0.053 (0.130) & 0.012 (0.106) & 0.029 (0.117)\\*
 &  &  & \ic{[-0.038, -0.021]} & \ic{[0.004, 0.021]} & \ic{[0.012, 0.027]} & \ic{[0.044, 0.063]} & \ic{[0.005, 0.023]} & \ic{[0.017, 0.036]}\\
 & 40 & 40 & \best{-0.021 (0.089)} & 0.022 (0.116) & 0.024 (0.109) & 0.082 (0.140) & 0.023 (0.111) & 0.038 (0.110)\\*
 &  &  & \ic{[-0.027, -0.015]} & \ic{[0.015, 0.031]} & \ic{[0.016, 0.032]} & \ic{[0.069, 0.092]} & \ic{[0.016, 0.028]} & \ic{[0.029, 0.047]}\\
MAR & 10 & 10 & \best{-0.004 (0.107)} & 0.053 (0.132) & 0.051 (0.135) & 0.096 (0.155) & 0.056 (0.136) & 0.061 (0.140)\\*
 &  &  & \ic{[-0.012, 0.005]} & \ic{[0.044, 0.063]} & \ic{[0.041, 0.062]} & \ic{[0.085, 0.109]} & \ic{[0.046, 0.071]} & \ic{[0.055, 0.071]}\\
 & 25 & 25 & \best{0.017 (0.115)} & 0.059 (0.146) & 0.073 (0.144) & 0.118 (0.164) & 0.108 (0.186) & 0.098 (0.149)\\*
 &  &  & \ic{[0.006, 0.025]} & \ic{[0.047, 0.076]} & \ic{[0.060, 0.085]} & \ic{[0.104, 0.131]} & \ic{[0.094, 0.127]} & \ic{[0.088, 0.107]}\\
 & 40 & 40 & \best{0.057 (0.127)} & 0.098 (0.161) & 0.111 (0.154) & 0.184 (0.182) & 0.159 (0.181) & 0.167 (0.191)\\*
 &  &  & \ic{[0.046, 0.068]} & \ic{[0.089, 0.109]} & \ic{[0.098, 0.126]} & \ic{[0.173, 0.196]} & \ic{[0.142, 0.174]} & \ic{[0.153, 0.186]}\\
MNAR & 10 & 10 & \best{-0.022 (0.070)} & 0.009 (0.084) & 0.012 (0.090) & 0.008 (0.102) & 0.020 (0.076) & 0.016 (0.088)\\*
 &  &  & \ic{[-0.028, -0.015]} & \ic{[0.002, 0.015]} & \ic{[0.007, 0.019]} & \ic{[0.001, 0.015]} & \ic{[0.013, 0.026]} & \ic{[0.008, 0.023]}\\
 & 25 & 25 & 0.074 (0.119) & 0.061 (0.113) & 0.109 (0.133) & \best{0.031 (0.120)} & 0.222 (0.200) & 0.166 (0.167)\\*
 &  &  & \ic{[0.063, 0.085]} & \ic{[0.051, 0.068]} & \ic{[0.100, 0.123]} & \ic{[0.021, 0.042]} & \ic{[0.202, 0.235]} & \ic{[0.148, 0.183]}\\
 & 40 & 40 & 0.240 (0.180) & 0.219 (0.183) & 0.248 (0.200) & \best{0.038 (0.137)} & 0.792 (0.600) & 0.343 (0.223)\\*
 &  &  & \ic{[0.226, 0.255]} & \ic{[0.206, 0.231]} & \ic{[0.231, 0.262]} & \ic{[0.027, 0.050]} & \ic{[0.742, 0.839]} & \ic{[0.330, 0.357]}\\
\midrule
\multicolumn{9}{l}{\textit{Ozone}}\\*
MCAR & 10 & 10 & \best{-0.032 (0.098)} & 0.007 (0.114) & 0.006 (0.111) & 0.040 (0.166) & 0.006 (0.104) & 0.008 (0.105)\\*
 &  &  & \ic{[-0.041, -0.021]} & \ic{[-0.001, 0.017]} & \ic{[0.000, 0.012]} & \ic{[0.030, 0.049]} & \ic{[0.000, 0.014]} & \ic{[0.001, 0.015]}\\
 & 25 & 25 & \best{-0.050 (0.101)} & 0.009 (0.136) & 0.005 (0.140) & 0.052 (0.189) & 0.002 (0.121) & 0.012 (0.136)\\*
 &  &  & \ic{[-0.059, -0.039]} & \ic{[0.001, 0.022]} & \ic{[-0.006, 0.018]} & \ic{[0.034, 0.070]} & \ic{[-0.006, 0.011]} & \ic{[0.004, 0.020]}\\
 & 40 & 40 & \best{-0.040 (0.111)} & 0.014 (0.159) & 0.011 (0.130) & 0.055 (0.200) & 0.006 (0.131) & 0.020 (0.155)\\*
 &  &  & \ic{[-0.052, -0.032]} & \ic{[0.006, 0.022]} & \ic{[0.002, 0.020]} & \ic{[0.034, 0.078]} & \ic{[-0.001, 0.014]} & \ic{[0.002, 0.040]}\\
MAR & 10 & 10 & \best{-0.049 (0.095)} & -0.002 (0.115) & 0.001 (0.121) & 0.013 (0.156) & 0.000 (0.100) & -0.001 (0.125)\\*
 &  &  & \ic{[-0.054, -0.043]} & \ic{[-0.012, 0.009]} & \ic{[-0.006, 0.008]} & \ic{[0.003, 0.023]} & \ic{[-0.007, 0.008]} & \ic{[-0.009, 0.007]}\\
 & 25 & 25 & \best{-0.049 (0.124)} & 0.031 (0.160) & 0.019 (0.165) & 0.047 (0.194) & 0.012 (0.158) & 0.023 (0.174)\\*
 &  &  & \ic{[-0.060, -0.040]} & \ic{[0.020, 0.044]} & \ic{[0.004, 0.033]} & \ic{[0.032, 0.061]} & \ic{[0.000, 0.024]} & \ic{[0.010, 0.035]}\\
 & 40 & 40 & \best{-0.016 (0.136)} & 0.079 (0.197) & 0.068 (0.177) & 0.103 (0.212) & 0.043 (0.199) & 0.079 (0.190)\\*
 &  &  & \ic{[-0.025, -0.007]} & \ic{[0.059, 0.094]} & \ic{[0.052, 0.078]} & \ic{[0.082, 0.119]} & \ic{[0.029, 0.058]} & \ic{[0.052, 0.095]}\\
MNAR & 10 & 10 & \best{-0.009 (0.172)} & 0.061 (0.210) & 0.081 (0.227) & 0.022 (0.192) & 0.074 (0.234) & 0.076 (0.223)\\*
 &  &  & \ic{[-0.018, 0.004]} & \ic{[0.049, 0.078]} & \ic{[0.062, 0.092]} & \ic{[0.010, 0.036]} & \ic{[0.059, 0.082]} & \ic{[0.064, 0.089]}\\
 & 25 & 25 & 0.039 (0.219) & 0.085 (0.221) & 0.131 (0.238) & \best{0.028 (0.176)} & 0.114 (0.239) & 0.127 (0.250)\\*
 &  &  & \ic{[0.025, 0.051]} & \ic{[0.066, 0.102]} & \ic{[0.109, 0.143]} & \ic{[0.014, 0.047]} & \ic{[0.097, 0.130]} & \ic{[0.113, 0.146]}\\
 & 40 & 40 & 0.176 (0.227) & 0.235 (0.285) & 0.260 (0.269) & \best{0.066 (0.197)} & 0.243 (0.286) & 0.291 (0.307)\\*
 &  &  & \ic{[0.157, 0.188]} & \ic{[0.219, 0.256]} & \ic{[0.239, 0.277]} & \ic{[0.042, 0.078]} & \ic{[0.224, 0.262]} & \ic{[0.270, 0.316]}\\
\midrule
\multicolumn{9}{l}{\textit{Concrete}}\\*
MCAR & 10 & 10 & \best{-0.008 (0.072)} & 0.053 (0.087) & 0.055 (0.088) & 0.098 (0.100) & 0.054 (0.093) & 0.074 (0.100)\\*
 &  &  & \ic{[-0.017, 0.000]} & \ic{[0.046, 0.059]} & \ic{[0.050, 0.061]} & \ic{[0.090, 0.107]} & \ic{[0.048, 0.062]} & \ic{[0.064, 0.083]}\\
 & 25 & 25 & \best{0.064 (0.087)} & 0.146 (0.111) & 0.149 (0.105) & 0.199 (0.121) & 0.131 (0.133) & 0.182 (0.134)\\*
 &  &  & \ic{[0.058, 0.071]} & \ic{[0.135, 0.156]} & \ic{[0.139, 0.155]} & \ic{[0.187, 0.209]} & \ic{[0.121, 0.142]} & \ic{[0.174, 0.192]}\\
 & 40 & 40 & \best{0.173 (0.115)} & 0.243 (0.135) & 0.255 (0.137) & 0.317 (0.143) & 0.245 (0.146) & 0.309 (0.150)\\*
 &  &  & \ic{[0.157, 0.182]} & \ic{[0.233, 0.258]} & \ic{[0.241, 0.264]} & \ic{[0.309, 0.334]} & \ic{[0.232, 0.255]} & \ic{[0.297, 0.320]}\\
MAR & 10 & 10 & \best{0.097 (0.101)} & 0.154 (0.125) & 0.172 (0.122) & 0.225 (0.135) & 0.150 (0.110) & 0.208 (0.143)\\*
 &  &  & \ic{[0.090, 0.105]} & \ic{[0.143, 0.166]} & \ic{[0.163, 0.186]} & \ic{[0.214, 0.234]} & \ic{[0.143, 0.162]} & \ic{[0.196, 0.217]}\\
 & 25 & 25 & \best{0.245 (0.134)} & 0.294 (0.148) & 0.322 (0.153) & 0.351 (0.149) & 0.289 (0.135) & 0.405 (0.163)\\*
 &  &  & \ic{[0.232, 0.253]} & \ic{[0.280, 0.304]} & \ic{[0.314, 0.336]} & \ic{[0.344, 0.361]} & \ic{[0.281, 0.295]} & \ic{[0.397, 0.414]}\\
 & 40 & 40 & \best{0.413 (0.142)} & 0.464 (0.153) & 0.497 (0.169) & 0.521 (0.157) & 0.442 (0.154) & 0.593 (0.172)\\*
 &  &  & \ic{[0.400, 0.432]} & \ic{[0.450, 0.472]} & \ic{[0.480, 0.506]} & \ic{[0.513, 0.531]} & \ic{[0.424, 0.460]} & \ic{[0.580, 0.603]}\\
MNAR & 10 & 10 & \best{-0.011 (0.076)} & -0.004 (0.080) & 0.010 (0.092) & 0.086 (0.102) & 0.013 (0.082) & 0.052 (0.097)\\*
 &  &  & \ic{[-0.018, -0.006]} & \ic{[-0.010, 0.002]} & \ic{[0.004, 0.016]} & \ic{[0.079, 0.097]} & \ic{[0.010, 0.020]} & \ic{[0.041, 0.060]}\\
 & 25 & 25 & 0.120 (0.098) & 0.152 (0.111) & 0.142 (0.097) & \best{0.061 (0.092)} & 0.157 (0.105) & 0.211 (0.117)\\*
 &  &  & \ic{[0.112, 0.128]} & \ic{[0.144, 0.158]} & \ic{[0.136, 0.151]} & \ic{[0.054, 0.070]} & \ic{[0.151, 0.163]} & \ic{[0.204, 0.219]}\\
 & 40 & 40 & 0.249 (0.113) & 0.281 (0.127) & 0.227 (0.112) & \best{0.055 (0.095)} & 0.238 (0.115) & 0.358 (0.146)\\*
 &  &  & \ic{[0.242, 0.261]} & \ic{[0.272, 0.292]} & \ic{[0.220, 0.237]} & \ic{[0.048, 0.059]} & \ic{[0.229, 0.246]} & \ic{[0.350, 0.373]}\\
\end{longtable}}

\begin{figure}[!ht]
\centering
\includegraphics[width=\textwidth]{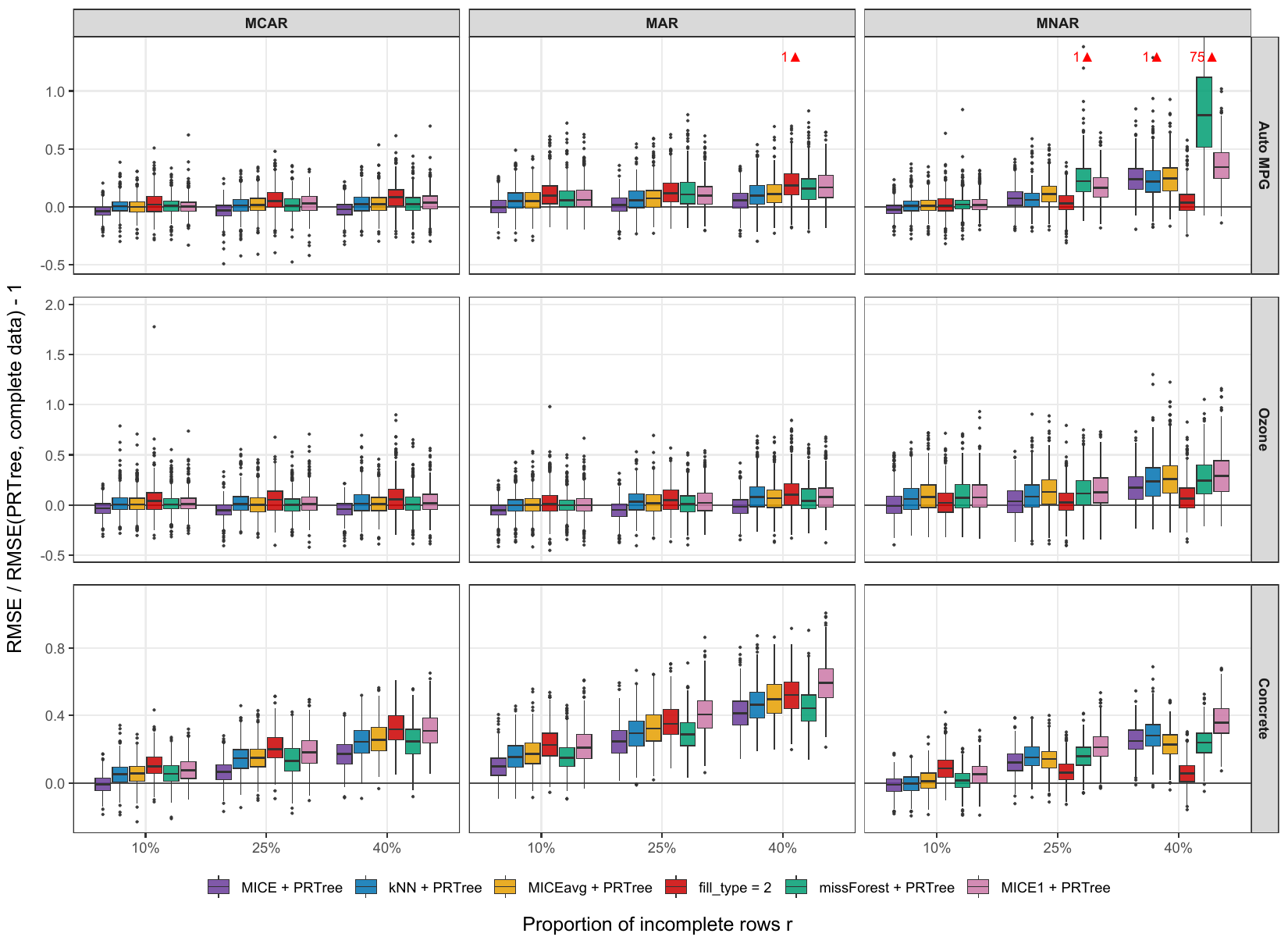}
\caption{Semi-synthetic study: distribution over replications and folds of the loss $L = \mathrm{RMSE}/\mathrm{RMSE}_{\texttt{PR-full}} - 1$ for the six procedures with the highest overall win rate, by dataset (rows), mechanism (columns) and rate $r$ (horizontal axis). The horizontal line at $0$ corresponds to PRTree fitted to the complete data. In Auto MPG, the vertical axis was truncated to improve readability; the red symbols $n\blacktriangle$ indicate the number $n$ of observations above the limit.}\label{fig:sim-box}
\end{figure}

\paragraph{Uncertainty over replications.} The 95\% confidence intervals in Table~\ref{tab:sim} are narrow: $84\%$ of the 162 intervals have width at most $0.03$ (median width $0.02$, maximum $0.097$). In 24 of the 27 scenarios, the interval of the procedure with the smallest median loss does not overlap with that of any of the other five procedures. The exceptions are Ozone under MNAR with $r = 25\%$ (\texttt{PR-f2} and \texttt{MICE+PR}), Concrete under MAR with $r = 40\%$ (\texttt{MICE+PR} and \texttt{missForest+PR}) and Concrete under MNAR with $r = 10\%$ (\texttt{MICE+PR} and \texttt{kNN+PR}).

\paragraph{Internal strategies.} As in Section~\ref{sec:aplic}, \texttt{PR-f2} is the most accurate of the three strategies. Considering the $27$ scenarios (three datasets, three mechanisms and three values of $r$), \texttt{PR-f2} has a smaller RMSE than \texttt{PR-f1} in at least $66\%$ of the folds in all of them, and than \texttt{PR-f0} in more than half of the folds (at least $53\%$) in $25$ of the $27$. The two exceptions have $r = 10\%$ (Concrete under MNAR and Ozone under MAR), where \texttt{PR-f2} wins against \texttt{PR-f0} in $24\%$ and $49\%$ of the folds, respectively. Compared with CART with surrogate splits, \texttt{PR-f2} has a smaller RMSE in at least $68\%$ of the folds in all $27$ scenarios.

\paragraph{Internal strategies versus impute-then-PRTree.} Since the base model is the same, this comparison isolates the effect of the treatment of missing values. Under MCAR and MAR, \texttt{PR-f2} is at least as accurate as mean imputation in most scenarios (it wins in $45\%$ to $83\%$ of the folds), but it is generally outperformed by the other imputation methods: it wins in $16\%$ to $45\%$ of the folds against \texttt{kNN+PR}, in $19\%$ to $50\%$ against \texttt{missForest+PR} and in only $1\%$ to $26\%$ against \texttt{MICE+PR}, which has the smallest median loss among the six procedures in all 18 MCAR and MAR scenarios (Table~\ref{tab:sim}). Under MNAR, where the highest values of each predictor are missing, the picture changes. The loss of \texttt{PR-f2} grows slowly with $r$ (the median is between $0.03$ and $0.07$ for $r \in \{25\%, 40\%\}$ in all datasets), whereas that of the imputation-based pipelines grows much faster. For $r \in \{25\%, 40\%\}$, \texttt{PR-f2} has the smallest median loss among the six procedures in all six scenarios (Table~\ref{tab:sim}), and it has a smaller RMSE than \texttt{kNN+PR} in $60\%$ to $100\%$ of the folds, than \texttt{missForest+PR} in $72\%$ to $99\%$ and than \texttt{MICE+PR} in $51\%$ to $99\%$. For $r = 10\%$, \texttt{MICE+PR} has the smallest median loss in all datasets. A possible explanation is that, when the missing values are the highest ones, the imputation methods replace them by values typical of the observed (lower) ones, thereby sending these observations to the wrong side of the splits on the affected predictors; \texttt{fill\_type} = 2 does not impute and, when the splitting coordinate is missing, divides the association equally between the two children, which appears to be less harmful in this case. This advantage was not observed when both tails are deleted (design A and B, above), which indicates that it depends on the direction of the deletion.

\paragraph{Multiple imputation and averaging.} \texttt{MICE+PR} averages the predictions of $m = 5$ trees, one for each imputed dataset, while all the other procedures use a single tree. Averaging trees fitted to slightly different datasets reduces the variance of the predictions, as in bagging. This explains why \texttt{MICE+PR} often has a negative loss, that is, a smaller RMSE than PRTree fitted to the complete data, which cannot be attributed to the imputation itself; this occurs in $69\%$ and $71\%$ of the folds in Auto MPG and Ozone under MCAR, and in $68\%$ of the folds in Ozone under MAR. The single-tree versions of MICE, which use the same imputations, show that averaging accounts for most of the advantage of \texttt{MICE+PR}. For instance, under MCAR with $r = 25\%$, the median loss increases from $-0.03$ (\texttt{MICE+PR}) to $0.03$ (\texttt{MICE$_1$+PR}) in Auto MPG, from $-0.05$ to $0.01$ in Ozone, and from $0.06$ to $0.18$ in Concrete. Against \texttt{MICE$_1$+PR}, \texttt{PR-f2} wins in $34\%$ to $74\%$ of the folds under MCAR and MAR and in $73\%$ to $100\%$ under MNAR with $r \in \{25\%, 40\%\}$. \texttt{MICE$_{\mathrm{avg}}$+PR} lies between the two MICE versions: it is more accurate than \texttt{PR-f2} under MCAR and MAR (\texttt{PR-f2} wins in $23\%$ to $47\%$ of the folds), whereas under MNAR with $r \in \{25\%, 40\%\}$ \texttt{PR-f2} wins in $70\%$ to $98\%$.

\paragraph{Impute-then-PRTree versus impute-then-CART.} For a given imputation method, the comparison isolates the effect of the base model. In most scenarios, PRTree is more accurate than CART fitted to the same imputed data: depending on the imputation method, the mechanism and $r$, it has a smaller RMSE in $61\%$ to $98\%$ of the folds in Auto MPG, $41\%$ to $96\%$ in Ozone and $14\%$ to $94\%$ in Concrete. The exceptions occur mainly under MNAR with $r \in \{25\%, 40\%\}$ (in Ozone with mean imputation and in Concrete with mean, $k$NN, missForest, MICE and MICE$_{\mathrm{avg}}$ imputation), where the lowest value is $14\%$ (Concrete, mean imputation, $r = 40\%$). Against the best impute-then-CART pipeline (\texttt{MICE+CART}), \texttt{PR-f2} wins in $49\%$ to $99\%$ of the folds in Auto MPG and $59\%$ to $89\%$ in Ozone, depending on the mechanism and $r$; in Concrete, it wins in $71\%$ to $96\%$ of the folds under MNAR with $r \in \{25\%, 40\%\}$, but in only $4\%$ to $25\%$ under MCAR and MAR.

\begin{figure}[!ht]
\centering
\includegraphics[width=\textwidth]{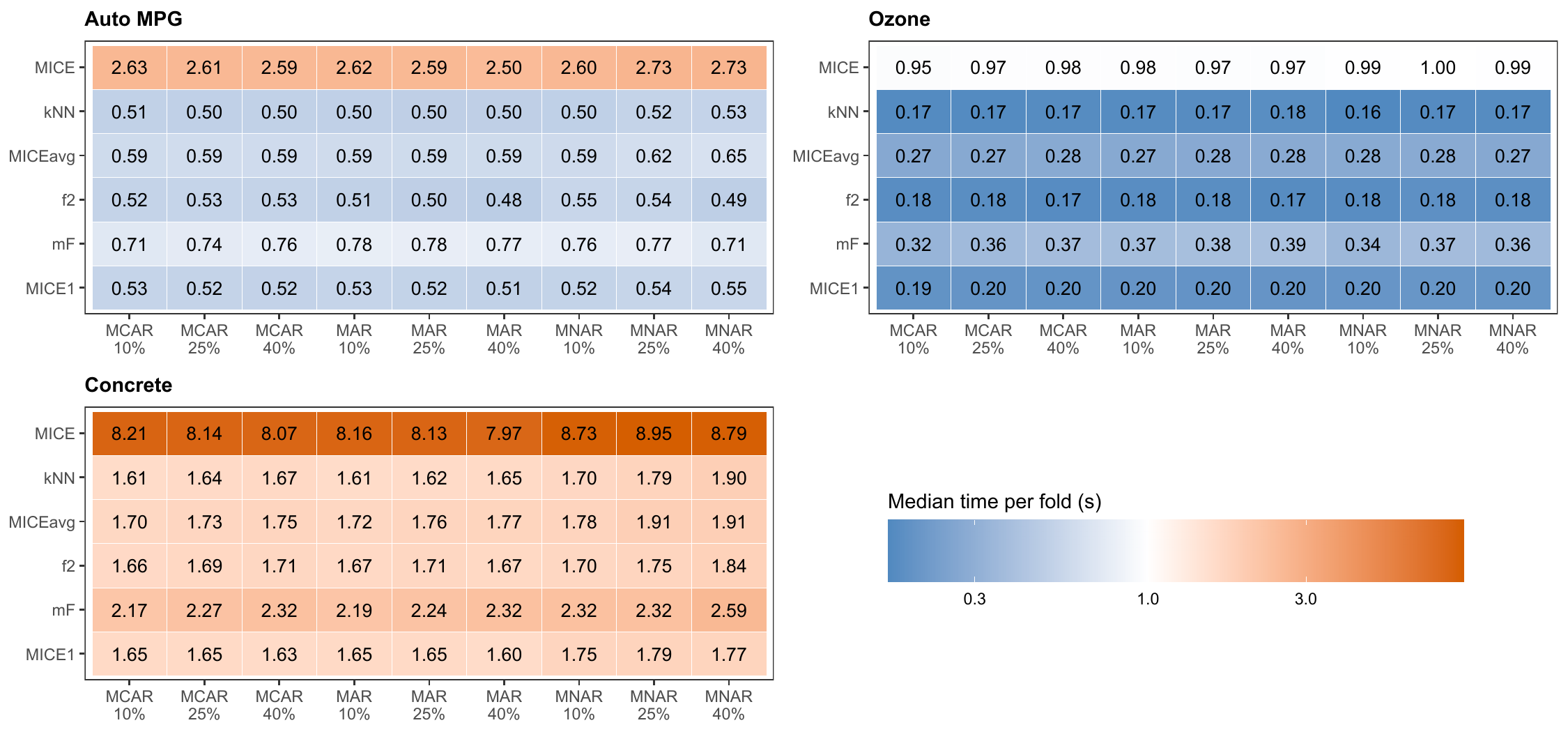}
\caption{Semi-synthetic study: median time per fold (in seconds), including imputation, model fitting and prediction, for the six procedures with the highest overall win rate, by dataset, mechanism and rate $r$. The tasks were run in parallel on the computer described above, so that absolute times are only indicative.}\label{fig:sim-time}
\end{figure}

\paragraph{Accuracy versus computational cost.} Figure~\ref{fig:sim-time} reports the median time per fold of the six procedures, including imputation, model fitting and prediction. For the sample sizes considered, imputation is cheap compared with fitting a PRTree: the pipelines based on $k$NN, missForest and MICE$_{\mathrm{avg}}$ imputation cost between $1.0$ and $2.1$ times as much as \texttt{PR-f2}, whereas \texttt{MICE+PR}, which fits one tree per imputed dataset, costs between $4.9$ and $5.5$ times as much. To weigh accuracy against cost, consider the gain $G = 1 - \mathrm{RMSE}/\mathrm{RMSE}_{\texttt{PR-f2}}$, computed fold by fold, so that positive values indicate a smaller RMSE than \texttt{PR-f2}. Under MCAR and MAR, the median gain of \texttt{MICE+PR} lies between $0.08$ and $0.11$, at five times the cost of \texttt{PR-f2}, while $k$NN imputation yields median gains between $0.02$ and $0.06$ at about the same cost, and missForest and MICE$_{\mathrm{avg}}$ imputation have similar median gains, between $0.02$ and $0.05$, at $1.0$ to $2.1$ times the cost. Under MNAR, the extra cost does not pay off: for $r \in \{25\%, 40\%\}$, the median gain of every imputation-based pipeline is negative, ranging from $-0.75$ to $-0.004$. Although the absolute times are small for the datasets considered (at most about $9$ seconds per fold), the five-fold increase in cost of multiple imputation would be relevant for larger datasets; for instance, the evaluation of PRTree on the Crime dataset in Section~\ref{sec:aplic} required more than 35 hours.

\paragraph{Summary.} The semi-synthetic study shows that the internal strategies, and \texttt{fill\_type} = 2 in particular, are a competitive alternative to impute-then-fit pipelines, but they do not dominate them. When the missingness is ignorable (MCAR or MAR), multiple imputation with averaging over the imputed datasets and $k$NN imputation lead to more accurate predictions than \texttt{fill\_type} = 2, as does missForest imputation in most scenarios, while a single tree fitted to one MICE imputation is only comparable to \texttt{fill\_type} = 2. When the highest values of the predictors are missing (MNAR), \texttt{fill\_type} = 2 has the smallest median loss in all scenarios with moderate or high missingness rates and outperforms every single-tree imputation pipeline in most folds, at a lower computational cost than multiple imputation; with $10\%$ of incomplete observations, imputation is still preferable. This advantage is specific to the one-sided deletion: when both tails are deleted, imputation followed by PRTree is more accurate (design A and B). In terms of accuracy versus cost, $k$NN imputation offers the best trade-off under MCAR and MAR, multiple imputation buys a larger reduction of the RMSE at five times the cost of \texttt{fill\_type} = 2, and, under one-sided MNAR with moderate or high rates, no imputation-based pipeline compensates for its additional cost. In all scenarios, \texttt{fill\_type} = 2 is more accurate than CART with surrogate splits in most folds. These conclusions are restricted to the datasets, mechanisms and PRTree specification considered, and the behavior under other forms of MNAR missingness may differ (Remark~\ref{rmk:mechanism}).

\FloatBarrier
\section{Conclusion}\label{sec:conc}

This paper introduced three different strategies to adapt probabilistic regression trees allowing it to handle missing predictor values directly during tree construction. Rather than relying on external imputation procedures or discarding incomplete observations, the proposed approach incorporates probabilistic mechanisms into the split-selection process, allowing incomplete observations to contribute to model estimation throughout the recursive partitioning procedure.

The empirical results show that PRTree can provide substantial predictive improvements over CART when missing predictor values affect a considerable fraction of the data, but not in every dataset: for Garment, CART was generally superior despite the large fraction of incomplete observations, a result that is consistent with the structural (and hence informative) missingness of that dataset, since the two methods are practically tied once the incomplete predictor is removed or the splitting constraint \texttt{prop\_x} is relaxed. The magnitude, and even the sign, of these  differences depends strongly on the adopted fill strategy and on the characteristics of the dataset. In the datasets with a large proportion of incomplete observations (Crime and Ozone), the fill mechanism emerged as the dominant modeling component, producing larger performance differences than those associated with either the proxy-selection criterion or the choice of smoothing distribution, which acted as refinement mechanisms. In the remaining datasets, the fill strategy had little influence on the win rates, and the smoothing distribution varied them most.

The experiments also revealed that the effectiveness of the probabilistic framework depends on the amount of missing values present in the data. Datasets with a large proportion of observations containing missing predictor values, such as Crime and Ozone, benefited substantially from the proposed methodology. Conversely, when missing values affected only a small fraction of the observations, as in Auto MPG, the differences among the fill strategies became considerably smaller, while PRTree still outperformed CART in most folds, which points to the smoothness of the base model rather than to the treatment of missing values. These findings suggest that the potential gains obtained from the choice of fill strategy are closely related to the extent to which incomplete observations participate in the learning process, although the amount of missingness alone does not determine whether PRTree outperforms CART, as illustrated by Garment.

The semi-synthetic study of Section~\ref{sec:sim} complements these findings by comparing the proposed strategies with impute-then-fit pipelines under controlled MCAR, MAR and MNAR mechanisms. With complete data, PRTree was more accurate than CART, which confirms that part of its advantage comes from the smoothness of the base model. Among the internal strategies, \texttt{fill\_type} = 2 was the most accurate in most scenarios (the exceptions are Concrete under MNAR and Ozone under MAR, both with $r = 10\%$, where \texttt{fill\_type} = 0 had the smallest median loss) and outperformed CART with surrogate splits in most folds. However, the internal strategies do not dominate imputation-based workflows: under MCAR and MAR, imputing the predictors with $k$NN or MICE, and in most cases with missForest, before fitting PRTree led to more accurate predictions, although most of the advantage of MICE came from averaging the predictions of several trees, at about five times the computational cost of \texttt{fill\_type} = 2, while $k$NN imputation offered the best trade-off between accuracy and cost. Under the MNAR mechanism considered, in which the highest values of the predictors are missing, \texttt{fill\_type} = 2 had the smallest median loss in all scenarios with moderate or high missingness rates and outperformed every single-tree imputation pipeline in most folds, but this advantage was not observed when both tails were deleted.

From a computational perspective, PRTree is substantially more demanding than CART. The additional cost arises from repeated tree reconstructions during smoothing-parameter selection and from the evaluation of multiple candidate splits using the full MSE criterion. The experiments indicate that computational cost is driven primarily by the interaction between sample size and the number of candidate splits selected for full evaluation. Nevertheless, the most computationally expensive configurations were not systematically associated with the best predictive performance, indicating that increased computational effort alone does not guarantee improved accuracy.

\paragraph{Limitations: categorical predictors.} The empirical evaluation is restricted to numerical predictors. As discussed in Section~\ref{sec:aplic}, this limitation is inherited from the base model, which, like other probabilistic trees, does not handle categorical predictors natively. Once categorical predictors are transformed into numerical columns, the three strategies apply without modification, because they are defined for arbitrary patterns of missing coordinates. With one-hot (dummy) coding, a missing category produces missing values in all dummy columns of that variable, a block pattern that is a particular case of the patterns covered by $\Psi^\ast$. Under \texttt{fill\_type} = 0 and 1, no further specification is needed, and the hard assignment of \texttt{fill\_type} = 1 is natural for binary coordinates. Under \texttt{fill\_type} = 2, the smoothing $\phi$ is applied to the observed dummy columns, and whether smoothing a binary coordinate with a continuous density is adequate is an open question. A native formulation of $\Psi$ for categorical coordinates, for instance through smoothing over the categories, together with an empirical evaluation on encoded categorical predictors, is left for future research.

\paragraph{Theoretical obstacles to consistency.} As stressed in Remark~\ref{rmk:scope}, \textbf{P1} and \textbf{P2} do not imply consistency. The consistency proof in \cite{alkhoury2020} relies on the fact that, in the complete-data case, $\Psi(\cdot\,;\Rg_m,\bs\sigma)$ is the convolution of the indicator of $\Rg_m$ with the scaled density $\phi$, whose Fourier transform does not vanish, so that smooth regression functions can be approximated by linear combinations of such functions. For incomplete observations, Lemma~\ref{lem:product} shows that $\Psi^\ast$ is a product of ratios of proxies, which is no longer of this form, and the obstacles differ across strategies. Under \texttt{fill\_type} = 0, $\Psi^\ast(\bs X;\Rg_m,\bs\sigma) = 2^{-d(m)}$ for every incomplete observation, where $d(m)$ is the depth of $\Rg_m$; the observed information is discarded, and the prediction is the same constant for all incomplete observations; hence, whenever incomplete observations have positive probability, it cannot be consistent for a regression function that varies with the observed coordinates among them, since no constant approximates a non-constant function in mean squared error. Under \texttt{fill\_type} = 1, the proxies are indicator functions of the observed coordinates, so that, for incomplete observations, the predictor has the hard discontinuities of a classical tree, and the smoothing argument cannot be used; consistency would require combining arguments for data-dependent hard partitions with those for smooth assignments. Under \texttt{fill\_type} = 2, the proxies are again convolutions of (projected) indicators with the marginal smoothing density, which makes this strategy the most promising starting point for a consistency result. Two difficulties remain common to all strategies. First, a split on a missing coordinate always divides the association equally between the children, whereas a consistent estimator would weight them according to the conditional distribution of the missing coordinate given the observed ones. Second, for incomplete observations the natural target becomes $\E(Y | \bs X_{|S}, \text{missingness pattern})$, which depends on the missingness mechanism \citep{josse2024}.

\paragraph{Extensions.} The construction of Section~\ref{sec:missing} is modular. Properties \textbf{P1} and \textbf{P2} hold for any smoothing density satisfying the assumptions of the base model, and any non-negative proxy $H$ satisfying $H(\bs X;\Rf*{L},\bs\sigma) + H(\bs X;\Rf*{R},\bs\sigma) > 0$ defines a valid allocation rule for incomplete observations, because the recursion in the definition of $\Psi^\ast$ renormalizes by the sum of the proxies of the two children. The proxy-selection criteria are used only to rank candidate splits and do not affect \textbf{P1} and \textbf{P2}. Therefore, new smoothing densities, proxy criteria and allocation mechanisms can be incorporated without modifying the theory of Section~\ref{sec:missing}. Several other directions for future research remain open. The current implementation uses the same smoothing distribution $\phi$ for all coordinates, and alternative specifications, such as a different $\phi$ for each coordinate or other fill strategies, may further improve predictive performance. Extensions to ensemble methods, such as random forests, constitute a particularly promising direction, as the probabilistic framework proposed here is naturally compatible with bootstrap aggregation and randomized split selection. Additional work is also needed to investigate the behavior of the strategies under a wider range of missing-data mechanisms and to develop more computationally efficient procedures for candidate-split evaluation.

Overall, the results indicate that probabilistic tree construction provides a viable alternative for regression problems involving incomplete predictor information. By integrating missing-value handling directly into the tree-building process, PRTree offers a flexible framework that preserves the interpretability of regression trees and does not require a separate imputation step. The proposed strategies should be regarded as an additional modeling option rather than as a replacement for imputation-based workflows, and, as for any method for incomplete data, their performance depends on the missingness mechanism.

\subsection*{CRediT authorship contribution statement}
\textbf{Taiane Schaedler Prass:} Conceptualization, Data curation, Formal analysis, Funding acquisition, Investigation, Methodology, Project administration, Resources, Software, Supervision, Validation, Visualization, Writing -- original draft, Writing -- review \& editing. \textbf{Alisson Silva Neimaier:} Conceptualization, Formal analysis, Investigation, Resources, Software, Writing -- original draft. \textbf{Guilherme Pumi:} Formal analysis, Funding acquisition, Methodology, Resources, Writing -- review \& editing.

\subsection*{Acknowledgments}
G. Pumi and T.S. Prass gratefully acknowledge the financial support received by the Conselho Nacional de Desenvolvimento Cient\'ifico e Tecnol\'ogico -- CNPq Brasil  -- Bolsa de Produtividade em Pesquisa - Proc. 303281/2025-1 (Pumi) and 305886/2025-8 (Prass). A.S. Neimaier gratefully acknowledge the financial support received by the Coordena\c{c}\~ao de Aperfei\c{c}oamento de Pessoal de N\'ivel Superior -- Brasil (CAPES) -- Finance Code 001.

\subsection*{Declaration of generative AI and AI-assisted technologies in the manuscript preparation process}
During the preparation of this work the authors used Claude (Anthropic). The system was used to support language editing, code development and data visualization. After using this tool, the authors reviewed and edited the content as needed and take full responsibility for the content of the published article.

\bibliographystyle{elsarticle-num-names}
\bibliography{biblio_rev2}
\end{document}